\documentclass{article}

\usepackage[nonatbib,preprint]{neurips_2024}
\usepackage{breakurl}

\usepackage[utf8]{inputenc} % allow utf-8 input
\usepackage[T1]{fontenc}    % use 8-bit T1 fonts
\usepackage{hyperref}       % hyperlinks
\usepackage{url}            % simple URL typesetting
\usepackage{booktabs}       % professional-quality tables
\usepackage{amsfonts}       % blackboard math symbols
\usepackage{nicefrac}       % compact symbols for 1/2, etc.
\usepackage{microtype}      % microtypography
\usepackage{xcolor}         % colors

\usepackage[square,sort&compress,comma,numbers]{natbib}
\usepackage{amsmath}
\usepackage{amssymb}
\usepackage{mathtools}
\usepackage{amsthm}

\usepackage[capitalize,noabbrev]{cleveref}

\theoremstyle{plain}
\newtheorem{theorem}{Theorem}[section]

\newtheorem{lemma}[theorem]{Lemma}

\theoremstyle{definition}
\newtheorem{definition}[theorem]{Definition}

\theoremstyle{remark}
\newtheorem{remark}[theorem]{Remark}

\theoremstyle{plain}

\usepackage[commandRef=cref]{proof-at-the-end}
\newEndThm[end, text link=]{theoremE}{theorem}
\newEndThm[end, text link=]{corollaryE}{corollary}
\newEndThm[end, text link=]{statementE}{statement}
\newEndThm[end, text link=]{lemmaE}{lemma}
\newEndThm[end, text link=]{remarkE}{remark}
\newEndProof{proofE}

\usepackage[shell, subfolder]{gnuplottex}
\usepackage{gnuplot-lua-tikz}

\usepackage{caption}
\usepackage{subcaption}

\usepackage{makecell}

\usepackage{algorithm}
\usepackage{algorithmic}

\def\PDF{p}
\def\defeq{\triangleq}

\DeclareMathOperator{\proba}{\mathbb{P}}
\DeclareMathOperator{\expect}{\mathbb{E}} % Математическое ожидание.
\DeclareMathOperator{\cov}{\textnormal{cov}}   % Ковариация.
\DeclareMathOperator{\supp}{\textnormal{supp}} % Носитель меры.
\DeclareMathOperator{\diag}{diag}

\def\limproba{\underset{N \to \infty}{\overset{\proba}{\longrightarrow}}}
\def\almosteq{\overset{\textnormal{a.s.}}{=}}

\DeclareMathOperator{\PMI}{\textnormal{PMI}}

\DeclareMathOperator*{\DKLoperator}{D_{\textnormal{KL}}}
\newcommand{\DKL}[2]{\DKLoperator \left( #1 \, || \, #2 \right)}

\def\loglikelihood{\mathcal{L}}

\def\reals{\mathbb{R}}

\def\normal{\mathcal{N}}
\def\uniform{\textnormal{U}}

\newcommand\undermat[2]{%
    \makebox[0pt][l]{$\smash{\underbrace{\phantom{%
        \begin{matrix}#2\end{matrix}}}_{\text{$#1$}}}$
    }#2%
}

\DeclareMathOperator{\arcsinh}{arcsinh}

\title{Mutual Information Estimation via Normalizing Flows}

\author{
Butakov I.~D. \\
Skoltech\thanks{Skolkovo Institute of Science and Technology},
MIPT\thanks{Moscow Institute of Physics and Technology},
Sirius\thanks{Sirius University of Science and Technology},\\
\texttt{butakov.id@phystech.edu} \\
\And
Tolmachev A.~D.\\
Skoltech, MIPT\\
\texttt{tolmachev.ad@phystech.edu} \\
\And
Malanchuk S.~V. \\
Skoltech, MIPT \\
\texttt{malanchuk.sv@phystech.edu} \\
\And
Neopryatnaya A.~M. \\
Skoltech, MIPT \\
\texttt{neopryatnaya.am@phystech.edu} \\
\And
Frolov A.~A. \\
Skoltech \\
\texttt{al.frolov@skoltech.ru} \\
}

\begin{document}

\maketitle

\begin{abstract}
We propose a novel approach to the problem of \emph{mutual information} (MI) estimation via introducing a family of estimators based on normalizing flows. The estimator maps original data to the target distribution, for which MI is easier to estimate. We additionally explore the target distributions with known closed-form expressions for MI. Theoretical guarantees are provided to demonstrate that our approach yields MI estimates for the original data. Experiments with high-dimensional data are conducted to highlight the practical advantages of the proposed method.
\end{abstract}

\section{Introduction}\label{section:introduction}

% Applications.
Information-theoretic analysis of deep neural networks (DNN) has attracted recent interest due to intriguing fundamental results and new hypothesis.
Applying information theory to DNNs may provide novel tools for explainable AI via estimation of information flows~\cite{tishby2015bottleneck_principle, xu2017IT_analysis, goldfeld2019estimating_information_flow, abernethy2020reasoning_conditional_MI, kairen2018individual_neurons},
as well as new ways to encourage models to extract and generalize information~\cite{tishby2015bottleneck_principle, chen2016infogan, belghazi2018mine, ardizzone2020training_normflows}.
Useful applications of information theory to the classical problem of independence testing are also worth noting~\cite{berrett2017independence_testing, sen2017conditional_independence_test, duong2023normflows_for_conditional_independence_testing}.

Most of the information theory applications to the field of machine learning are based on the two central information-theoretic quantities: \emph{differential entropy} and \emph{mutual information} (MI).
The latter quantity is widely used as an invariant measure of the non-linear dependence between random variables,
while differential entropy is usually viewed as a measure of randomness.
However, as it has been shown in the previous works~\cite{goldfeld2020convergence_of_SEM_entropy_estimation, mcallester2020limitations_MI}, MI and differential entropy are extremely hard to estimate in the case of high-dimensional data.
It is argued that such estimation is also challenging for long-tailed distributions and large values of MI~\cite{czyz2023beyond_normal}.
These problems considerably limit the applications of information theory to real-scale machine learning problems.
However, recent advances in the neural estimation methods show that complex parametric estimators achieve relative practical success in the cases where classical MI estimation techniques fail~\cite{belghazi2018mine, oord2019representation_learning_CPC, song2020understanding_limitations, rhodes2020telescoping, Ao_Li_2022entropy_estimation_normflows, butakov2024lossy_compression, franzese2024minde}.

% Our and other solutions.
This paper addresses the mentioned problem of the mutual information estimation in high dimensions via using normalizing flows~\cite{tabak2010PDF_estimators_LL_ascend, tabak2013nonparametric_PDF_estimators, dinh2015NICE, rezende2015variational_inference_normflows}.
Some recent works also utilize generative models to estimate MI.
%In such articles, the estimate is acquired by exploiting the relation between MI and either differential entropy or the Kullback–Leibler (KL) divergence.
According to a general generative approach described in~\cite{song2020understanding_limitations}, generative models can be used to reconstruct probability density functions (PDFs) of marginal and joint distributions to estimate differential entropy and MI via a Monte Carlo (MC) integration.
However, as it is mentioned in the original work, when flow-based generative models are used, this approach yields poor estimates even when the data is of simple structure.
This approach is further investigated in~\cite{duong2023dine}.
The estimator proposed in~\cite{franzese2024minde} uses score-based diffusion models to estimate the differential entropy and MI without an explicit reconstruction of the PDFs.
Increased accuracy of this estimator comes at a cost of training score networks and using them to compute an MC estimate of Kullback–Leibler divergence (KLD).
Finally, in~\cite{duong2023normflows_for_conditional_independence_testing} normalizing flows are used to transform marginal distributions into Gaussian distributions, after which a zero-correlation criterion is employed to test the zero-MI hypothesis.
%The latter method does not require MC integration, but also does not yield a numerical MI estimate.
The same idea is later used in~\cite{duong2023dine} (see DINE-Gaussian) to acquire an MI estimate, but no corresponding error bounds are possible to derive, as knowing marginal distributions only is insufficient to calculate the MI (see~\cref{remark:MI_after_Gaussianization_bounds_tightness}), which makes this estimator substantially flawed.
We also note the work,
where normalizing flows are combined with a $ k $-NN entropy estimator~\cite{Ao_Li_2022entropy_estimation_normflows}.

In contrast, our method allows for simplified (cheap and low-variance MC integration is required) or even \emph{direct} (i.e., no MC integration, nearest neighbors search or other similar data manipulations are required) and the accurate MI estimation with asymptotic and non-asymptotic error bounds.
Our contributions in this work are the following:
\begin{enumerate}
    \item
    We propose a MI-preserving technique to simplify the joint distribution of two random vectors (RVs) via a Cartesian product of trainable normalizing flows in order to facilitate the MI estimation. Non-asymptotic error bounds are provided, with the gap approaching zero under certain commonly satisfied assumptions, showing that our estimator is consistent.
    \item
    We suggest restricting the proposed MI estimator to allow for a \emph{direct} MI calculation via a \emph{simple closed-form formula}.
    We further refine our approach to require only $ O(d) $ additional learnable parameters to estimate the MI (here $ d $ denotes the dimension of the data).
    We provide additional theoretical and statistical guarantees for our restricted estimator: variance and non-asymptotic error bounds are derived.
    \item
    We validate and evaluate our method via experiments with high-dimensional synthetic data with known ground truth MI.
    We show that the proposed MI estimator performs well in comparison to the ground truth and some other advanced estimators during the tests with high-dimensional compressible and incompressible data of various complexity.
\end{enumerate}

This article is organized as follows.
In \cref{section:preliminaries}, the necessary background is provided and the key concepts of information theory are introduced.
\cref{section:method_description} describes the general method and corresponding theoretical results.
In \cref{section:restricted_method} we restrict our method to allow for accurate MI estimation via a closed-form formula.
%and contains results, useful for practical implementation.
In \cref{section:experiments}, a series of experiments is performed to evaluate the proposed method and compare it to several other key MI estimators.
Finally, the results are discussed in \cref{section:discussion}.
We provide all the proofs in~\cref{appendix:proofs},
technical details in~\cref{appendix:technical_details}.

%Recent works show relative practical success of the NN-based estimators
%This work addresses the problem of the estimation of the mutual information in high dimensions. For a given random vectors $X \in \mathbb{R}^{D_1}$ and $Y \in \mathbb{R}^{D_2}$, the mutual information was introduced (see [...]) as: $I(X, Y) = h(X) + H(Y) - H(X, Y)$. Our main goal is to estimate precisely this value $I(X, Y)$ based on samples $(X, Y)$ from the joint distribution. Existing estimators for this task based on the density estimation for the entropy evaluation (see \cite{goldfeld2019estimating_information_flow, kristjan2023smoothed_entropy_PCA}). The density of some unknown distribution could be estimated via several ways: kernel density estimation (see [...]), 

\section{Preliminaries}
\label{section:preliminaries}

Consider random vectors, denoted as $ X \colon \Omega \rightarrow \mathbb{R}^n $ and $ Y \colon \Omega \rightarrow \mathbb{R}^m $, where $ \Omega $ represents the sample space.
Let us assume that these random vectors are absolutely continuous, having probability density functions (PDF) denoted as $ \PDF(x) $, $ \PDF(y) $, and $ \PDF(x, y) $, respectively, where the latter refers to the joint PDF.
The differential entropy of $X$ is defined as follows:
\[
    h(X) = -\expect \log \PDF(x) = - \int\limits_{\mathclap{\supp{X}}} \PDF(x) \log \PDF(x) \, d x,
\]
where $ \supp{X} \subseteq \mathbb{R}^{n} $ represents the \emph{support} of $ X $, and $ \log(\cdot) $ denotes the natural logarithm.
Similarly, we define the joint differential entropy as $ h(X, Y) = -\expect \log \PDF(x, y) $ and conditional differential entropy as $ h(X \mid Y) = -\expect \log \PDF\left(X \middle| Y\right) = - \expect_{Y} \left(\expect_{X \mid Y = y} \log \PDF(X \mid Y = y) \right) $.
Finally, the mutual information (MI) is given by $ I(X; Y) = h(X) - h(X \mid Y) $, and the following equivalences hold
\begin{equation}
    \label{eq:mutual_information_from_conditional_entropy}
    I(X; Y) = h(X) - h(X\mid Y) = h(Y) - h(Y \mid X),
\end{equation}
\begin{equation}
    \label{eq:mutual_information_from_joined_entropy}
    I(X; Y) = h(X) + h(Y) - h(X, Y),
\end{equation}
\begin{equation}
    \label{eq:mutual_information_from_KullbackLeibler}
    I(X; Y)
    %= \expect \log \frac{\PDF_{X,Y}(x,y)}{\PDF_X(x) \PDF_Y(y)}
    = \DKL{\PDF_{X,Y}}{\PDF_X \otimes \PDF_Y}
\end{equation}
Mutual information can also be defined as an expectation of the \emph{pointwise mutual information}:
\begin{equation}
    \label{eq:pointwise_mutual_information}
    \PMI_{X,Y}(x,y) = \log \left[ \frac{\PDF(x \mid y)}{\PDF(x)} \right], \quad I(X;Y) = \expect \PMI_{X,Y}(X, Y)
\end{equation}
The above definitions can be generalized via Radon-Nikodym derivatives and induced densities in case of distributions supports being manifolds, see~\cite{spivak1965calculus}.

% здесь ссылку на нашу работу 2021 года надо добавить бы (см. ниже)

The differential entropy estimation is a separate classical statistical problem.
Recent works have proposed several novel ways to acquire the estimate in the high-dimensional case~\cite{berrett2019efficient_knn_entropy_estimation, goldfeld2020convergence_of_SEM_entropy_estimation, butakov2021high_dimensional_entropy_estimation, Ao_Li_2022entropy_estimation_normflows, kristjan2023smoothed_entropy_PCA, adilova2023IP_dropout, franzese2024minde}.
Due to~\cref{eq:mutual_information_from_joined_entropy}, mutual information can be found by estimating entropy values separately.
In contrast, this paper suggests an approach that estimates MI values directly.

In our work, we rely on the well-known fundamental property of MI,
which is invariance under smooth injective mappings.
The following theorem appears in literature in slightly different forms~\cite{kraskov2004KSG, czyz2023beyond_normal, czyz2023pointwise_MI, butakov2024lossy_compression, polyanskiy2024information_theory}; we utilize the one, which is the most convenient to use with normalizing flows.
\begin{theoremE}%[Statement 1 in~\cite{butakov2024lossy_compression}]
    \label{theorem:MI_under_nonsingular_mappings}
    Let $ \xi \colon \Omega \rightarrow \reals^{n'} $ be an absolutely continuous random vector, and let
    $ g \colon \mathbb{R}^{n'} \rightarrow \reals^n $ be an injective piecewise-smooth mapping with Jacobian $ J $,
    satisfying $ n \geq n' $ and
    $ \det \left( J^T J \right) \neq 0 $ almost everywhere.
    Let PDFs $ \PDF_\xi $ and $ \PDF_{\xi \mid \eta} $ exist.
    Then
    \begin{equation}
        \label{eq:MI_under_nonsingular_mapping}
        \PMI_{\xi,\eta}(a,b) \almosteq \PMI_{g(\xi),\eta}(g(a),b), \quad I(\xi; \eta) = I\left(g(\xi); \eta \right)
    \end{equation}
\end{theoremE}

\begin{proofE}
    For any function $ g $, let us denote $ \sqrt{\det \left( J^T(x) J(x) \right)} $ (area transformation coefficient) by $ \alpha(x) $ where it exists.
    
    Foremost, let us note that in both cases, $ \PDF_\xi(x \mid \eta) $ and $ \PDF_{g(\xi)}(x' \mid \eta) = \PDF_\xi(x \mid \eta) / \alpha(x) $ exist. Hereinafter, we integrate over $ \supp \xi \cap \left\{ x \mid \alpha(x) \neq 0 \right\} $ instead of $ \supp \xi $; as $ \alpha \neq 0 $ almost everywhere by the assumption, the values of the integrals are not altered.
    
    According to the definition of the differential entropy,
    \begin{align*}
        h(g(\xi)) &= -\int \frac{\PDF_\xi(x)}{\alpha(x)}\log\left(\frac{\PDF_\xi(x)}{\alpha(x)}\right)\alpha(x) \, dx = \\
        &= -\int \PDF_\xi(x) \log\left(\PDF_\xi(x)\right) dx + \int \PDF_\xi(x) \log\left(\alpha(x)\right) dx = \\
        & = h(\xi) + \expect \log \alpha(\xi).
    \end{align*}
    \begin{align*}
        h(g(\xi) \mid \eta) &= \expect_{\eta} \left( - \int \frac{\PDF_\xi(x \mid \eta)}{\alpha(x)} \log \left( \frac{\PDF_\xi(x \mid \eta)}{\alpha(x)} \right) \, \alpha(x) \, dx \right) = \\
        & = \expect_{\eta} \left( - \int \PDF_\xi(x \mid \eta) \log \left( \PDF_\xi(x \mid \eta) \right) dx + \int \PDF_\xi(x \mid \eta) \log \left( \alpha(x) \right) dx \right) = \\
        & = h(\xi \mid \eta) + \expect \log \alpha(\xi) \\
    \end{align*}
    Finally, by the MI definition,
    \[
        I(g(\xi); \eta) = h(g(\xi)) - h(g(\xi) \mid \eta) = h(\xi) - h(\xi \mid \eta) = I(\xi;\eta).
    \]

    Dropping the expectations/integrals in the equations above yields the proof of the $ \PMI $ invariance.
\end{proofE}

In our work, we heavily rely on the \emph{normalizing flows}~\cite{dinh2015NICE, rezende2015variational_inference_normflows}~-- trainable smooth bijective mappings with tractable Jacobian. However, to understand our results, it is sufficient to know that flow models (a) satisfy the conditions on $ g $ in \cref{theorem:MI_under_nonsingular_mappings} \emph{by definition}, (b) can model any absolutely continuous Borel probability measure (\emph{universality} property) and (c) are trained via a likelihood maximization, which is equivalent to a Kullback-Leibler divergence minimization.
For more details, we refer the reader to a more complete and rigorous overview of normalizing flows provided in~\cite{kobyzev2021normflows_overview}.
%We stress the importance of \cref{theorem:MI_under_nonsingular_mappings},
%as $ g $ can be interpreted as a forward (from latent space to data space) flow transformation.

\section{General method}
\label{section:method_description}

Our task is to estimate $ I(X;Y) $, where $ X $, $ Y $ are random vectors.
Here we focus on the absolutely continuous $ (X,Y) $, as it is the most common case in practice.
Note that \cref{theorem:MI_under_nonsingular_mappings} allows us to train normalizing flows $ f_X, f_Y $, apply them to $ X $, $ Y $ and consider estimating MI between the latent representations,
as $ I(f_X(X); f_Y(Y)) = I(X;Y) $.% due to the smoothness and bijectivity of flow models.

The key idea of our method is to train $ f_X $ and $ f_Y $ in such a way that $ I(f_X(X); f_Y(Y)) $ is easy to estimate.
For example, one can hope to acquire tractable pointwise mutual information (PMI), which can be then averaged via MC integration~\cite{czyz2023pointwise_MI}.
Unfortunately, the PMI invariance (\cref{theorem:MI_under_nonsingular_mappings}) restricts the possible distributions of $ (f_X(X), f_Y(Y)) $ to an unknown family,
making the exact MI recovery via such technique unfeasible.
%(as one, basically, has to know the PMI of $ (X,Y) $).
%Thus, tractable PMI in the latent space is generally not achievable, as it is equivalent to $ \PMI_{X,Y} $ also being tractable.

However, one can always approximate the PDF in latent space via a (preferably, simple) model $ q \in \mathcal{Q} $ with tractable PMI, and train $ q $, $ f_X $ and $ f_Y $ to minimize the discrepancy between the real and the proposed PMI. The complexity of $ q $ serves as a tradeoff: by selecting a poor $ \mathcal{Q} $, one might experience a considerable bias of the estimate; on the other hand, choosing $ \mathcal{Q} $ to be a universal PDF approximation family, one acquires a consistent, but computationally expensive MI estimate. Flows $ f_X, f_Y $ are used to tighten the approximation bound. We formalize this intuition in the following theorems:

\begin{theoremE}
    \label{theorem:MI_approximation_decomposition}
    Let $ (\xi, \eta) $ be absolutely continuous with PDF $ \PDF_{\xi,\eta} $.
    Let $ q_{\xi,\eta} $ be a PDF defined on the same space as $ \PDF_{\xi,\eta} $.
    Let $ \PDF_\xi $, $ \PDF_\eta $, $ q_\xi $ and $ q_\eta $ be the corresponding marginal PDFs.
    Then
    \begin{equation}
        \label{eq:MI_approximation_decomposition}
        I(\xi;\eta) = \underbrace{\expect_{\proba_{\xi,\eta}} \log \left[ \frac{q_{\xi,\eta}(\xi,\eta)}{q_{\xi}(\xi), q_{\eta}(\eta)} \right]}_{I_q(\xi; \eta)} + \DKL{\PDF_{\xi,\eta}}{q_{\xi,\eta} } - \DKL{\PDF_\xi \otimes \PDF_\eta}{q_\xi \otimes q_\eta}
    \end{equation}
\end{theoremE}

\begin{proofE}
    In the following text, all the expectations are in terms of $ \proba_{\xi,\eta} $.
    \begin{align*}
        I(\xi;\eta) &= \expect \log \left[ \frac{\PDF_{\xi,\eta}(\xi,\eta)}{\PDF_\xi(\xi) \PDF_\eta(\eta)} \right] = \expect \log \left[ \frac{q_{\xi,\eta}(\xi,\eta)}{q_\xi(\xi) q_\eta(\eta)} \cdot \frac{\PDF_{\xi,\eta}(\xi,\eta)}{q_{\xi,\eta}(\xi,\eta)} \cdot \frac{q_\xi(\xi) q_\eta(\eta)}{\PDF_\xi(\xi) \PDF_\eta(\eta)} \right] =\\
        &= I_q(\xi;\eta) + \expect \log \left[ \frac{\PDF_{\xi,\eta}(\xi,\eta)}{q_{\xi,\eta}(\xi,\eta)} \right] + \expect \log \left[ \frac{q_\xi(\xi)}{\PDF_\xi(\xi)} \right] + \expect \log \left[ \frac{q_\eta(\eta)}{\PDF_\eta(\eta)} \right] = \\
        &= I_q(\xi;\eta) + \DKL{\PDF_{\xi,\eta}}{q_{\xi,\eta} } - \DKL{\PDF_\xi \otimes \PDF_\eta}{q_\xi \otimes q_\eta}
    \end{align*}
\end{proofE}

\begin{corollaryE}
    \label{corollary:MI_approximation_decomposition_bounds}
    Under the assumptions of~\cref{theorem:MI_approximation_decomposition}, $ |I(\xi;\eta) - I_q(\xi;\eta)| \leq \DKL{\PDF_{\xi,\eta}}{q_{\xi,\eta} } $.
\end{corollaryE}

\begin{proofE}
    As $ \DKL{\PDF_\xi \otimes \PDF_\eta}{q_\xi \otimes q_\eta} \geq 0 $,
    \[
        I(\xi;\eta) \leq I_q(\xi,\eta) + \DKL{\PDF_{\xi,\eta}}{q_{\xi,\eta} }
    \]
    As $ \DKL{\PDF_{\xi,\eta}}{q_{\xi,\eta} } \geq \DKL{\PDF_\xi}{q_\xi} $ and $ \DKL{\PDF_{\xi,\eta}}{q_{\xi,\eta} } \geq \DKL{\PDF_\eta}{q_\eta} $ (monotonicity property, see Theorem 2.16 in~\cite{polyanskiy2024information_theory}),
    \[
        I(\xi;\eta) \geq I_q(\xi;\eta) + \DKL{\PDF_{\xi,\eta}}{q_{\xi,\eta} } - 2 \cdot \DKL{\PDF_{\xi,\eta}}{q_{\xi,\eta} } = I_q(\xi;\eta) - \DKL{\PDF_{\xi,\eta}}{q_{\xi,\eta} }
    \]
\end{proofE}

This allows us to define the following MI estimate:
\begin{equation}
    \label{eq:general_MI_estimate}
    \hat I_{\textnormal{MIENF}}(\{ (x_k, y_k) \}_{k=1}^N)%(X;Y)
    \defeq \hat I_{\hat q}(\hat f_X(X); \hat f_Y(Y)) = \frac{1}{N} \sum_{k=1}^N \log \left[ \frac{\hat q_{\xi,\eta}(\hat f_X(x_k), \hat f_Y(y_k))}{\hat q_\xi(\hat f_X(x_k)) \hat q_\eta(\hat f_Y(y_k))} \right],
\end{equation}
where $ \{ (x_k, y_k) \}_{k=1}^N $ is a sampling from $ (X,Y) $, and $ \hat q $, $ \hat f_X $ and $ \hat f_Y $ are selected according to the maximum likelihood.
The latter makes $ \hat I_{\textnormal{MIENF}} $ a consistent estimator:
\begin{theoremE}[$ \hat I_{\textnormal{MIENF}} $ is consistent]
    \label{eq:general_MI_estimate_consistency}
    Let $ X $, $ Y $, $ \hat f_X^{-1} $ and $ \hat f_Y^{-1} $ satisfy the conditions of~\cref{theorem:MI_under_nonsingular_mappings}.
    Let $ \mathcal{Q} $ be a family of universal PDF appproximators for a class of densities containing $ \proba_{X,Y} \circ (f_X^{-1} \times f_Y^{-1}) $ (pushforward probability measure in the latent space). Let $ \{ (x_k, y_k) \}_{k=1}^N $ be an i.i.d. sampling from $ (X, Y) $.
    Let $ \hat q_N \in \mathcal{Q} $ be a maximum-likelihood estimate of $ \proba_{X,Y} \circ (f_X^{-1} \times f_Y^{-1}) $ from the samples $ \{ (f_X(x_k), f_Y(y_k)) \}_{k=1}^N $.
    Let $ I_{\hat q_N}(f_X(X); f_Y(Y)) $ exist for every $ N $.
    Then
    \[
        \hat I_{\textnormal{MIENF}}(\{ (x_k, y_k) \}_{k=1}^N) \limproba I(X;Y)
    \]
\end{theoremE}

\begin{proofE}
    Following the assumptions on $ \hat q_N $, $ \DKL{\proba_{X,Y} \circ (f_X^{-1} \times f_Y^{-1})}{(\hat q_N)_{\xi,\eta} } \limproba 0 $ (\emph{universality} property).
    Due to~\cref{corollary:MI_approximation_decomposition_bounds}, this ensures $ I_{\hat q_N}(f_X(X); f_Y(Y)) \limproba I(f_X(X); f_Y(Y)) = I(X;Y) $ (the latter equality is due to~\cref{theorem:MI_under_nonsingular_mappings}).
    %As $ I(X;Y) = \expect \PMI_{X,Y}(X,Y) $ exists, and $ \PMI_{X,Y}(X,Y) \almosteq \PMI_{f_X(X),f_Y(Y)}(f_X(X),f_Y(Y)) $,
    Finally, $ \hat I_{\textnormal{MIENF}}(X;Y) \limproba I_{\hat q_N}(f_X(X); f_Y(Y)) $ as an MC estimate.
\end{proofE}

Note that maximum-likelihood training of $ f_X $, $ f_Y $ also minimizes $ \DKL{\proba_{X,Y} \circ (f_X^{-1} \times f_Y^{-1})}{\hat q_{\xi,\eta} } $, which allows for surprisingly simple $ q \in \mathcal{Q} $ to be used, as we show in the subsequent sections.

The described approach is as general as possible.
We use it as a starting point for a development of a more elegant, cheap and practically sound MI estimator.
We also do not incorporate conditions, under which the universality property of $ \mathcal{Q} $ holds, as they depend on the choice of $ \mathcal{Q} $;
if one is interested in using normalizing flows as $ \mathcal{Q} $,
we refer to Section 3.4.3 in~\cite{kobyzev2021normflows_overview} or to~\cite{duong2023dine} for more details.

\section{Using Gaussian base distribution}
\label{section:restricted_method}

Note that the general approach requires finding the maximum-likelihood estimate $ \hat q $ and using it to perform an MC integration to acquire $ \hat I_{\hat q}(f_X(X); f_Y(Y)) $.%, which increases the complexity of the method.

In this section, we drop these requirements by restricting our estimator via choosing $ \mathcal{Q} $ to be a family of multivariate Gaussian PDFs.
This allows as (a) to \emph{directly} estimate the MI via a \emph{closed-form} expression, (b) to employ a closed-form expression for optimal $ \hat q $, (c) to leverage the maximum entropy principle for Gaussian distributions, thus acquiring better non-asymptotic bounds, and (d) to analyze the variance of the proposed estimate.

\begin{theoremE}[Theorem 8.6.5 in~\cite{cover2006information_theory}]
    \label{theorem:Gaussian_maximizes_entropy}
    Let $ Z $ be a $ d $-dimensional absolutely continuous random vector with probability density function $ \PDF_Z $, mean $ m $ and covariance matrix $ \Sigma $.
    Then
    \begin{equation*}
        h(Z) = h\left( \normal(m, \Sigma) \right) - \DKL{\PDF_Z}{\normal(m, \Sigma)} 
        = \frac{d}{2} \log(2 \pi e) + \frac{1}{2} \log \det \Sigma - \DKL{\PDF_Z}{\normal(m, \Sigma)}
    \end{equation*}
\end{theoremE}

\begin{proofE}
    As $ h(Z - m) = h(Z) $,
    let us consider a centered random vector $ Z $.
    We denote probability density function of $ \normal(0, \Sigma) $ as $ \phi_\Sigma $.
    \[
        \DKL{\PDF_Z}{\normal(0, \Sigma)} = \int_{\reals^d} \PDF_Z(z) \log \frac{\PDF_Z(z)}{\phi_\Sigma(z)} \, dz = -h(Z) - \int_{\reals^d} \PDF_Z(z) \log \phi_\Sigma(z) \, dz
    \]
    Note that
    \[
        \int_{\reals^d} \PDF_Z(z) \log \phi_\Sigma(z) \, dz = const + \frac{1}{2} \expect_Z z^T \Sigma^{-1} z = const + \frac{1}{2} \expect_{\normal(0, \Sigma)} z^T \Sigma^{-1} z = \int_{\reals^d} \phi_\Sigma(z) \log \phi_\Sigma(z) \, dz,
    \]
    from which we arrive to the final result:
    \[
        \DKL{\PDF_Z}{\normal(0, \Sigma)} = -h(Z) + h(\normal(0, \Sigma))
    \]
\end{proofE}

\begin{remark}
    \label{remark:no_DKL_entropy_formula_in_general}
    Note that $ h(Z) \neq h(Z') - \DKL{\PDF_Z}{\PDF_{Z'}} $ in general.
\end{remark}

\begin{corollaryE}
    \label{corollary:MI_after_Gaussianization}
    Let $ (\xi,\eta) $ be an absolutely continuous pair of random vectors with joint and marginal probability density functions $ \PDF_{\xi,\eta} $, $ \PDF_{\xi} $ and $ \PDF_{\eta} $ correspondingly,
    and mean and covariance matrix being
    \[
        m =
        \begin{bmatrix}
            m_\xi \\
            m_\eta
        \end{bmatrix},
        \quad
        \Sigma =
        \begin{bmatrix}
            \Sigma_{\xi,\xi} & \Sigma_{\xi,\eta} \\
            \Sigma_{\eta,\xi} & \Sigma_{\eta,\eta}
        \end{bmatrix}
    \]
    Then
    \begin{multline*}
        %\label{eq:closed_form_MI_expression_Gaussian_with_corrections}
        I(\xi; \eta) = \frac{1}{2} \left[ \log \det \Sigma_{\xi,\xi} + \log \det \Sigma_{\eta,\eta} - \log \det \Sigma \right] + \\
        + \DKL{\PDF_{\xi,\eta}}{\normal(m, \Sigma)}
        - \DKL{\PDF_\xi \otimes \PDF_\eta}{\normal(m, \diag(\Sigma_{\xi,\xi}, \Sigma_{\eta,\eta})},
    \end{multline*}
    which implies the following in the case of marginally Gaussian $ \xi $ and $ \eta $:
    \begin{equation}
        \label{eq:closed_form_MI_expression_Gaussian}
        I(\xi; \eta) \geq \frac{1}{2} \left[ \log \det \Sigma_{\xi,\xi} + \log \det \Sigma_{\eta,\eta} - \log \det \Sigma \right],
    \end{equation}
    with the equality holding if and only if $ (\xi, \eta) $ are jointly Gaussian.
\end{corollaryE}

\begin{proofE}
    By applying~\cref{theorem:Gaussian_maximizes_entropy} to~\cref{eq:mutual_information_from_joined_entropy}, we derive the following expression:
    \begin{multline*}
        I(\xi; \eta) = \frac{1}{2} \left[ \log \det \Sigma_{\xi,\xi} + \log \det \Sigma_{\eta,\eta} - \log \det \Sigma \right] + \\
        + \DKL{\PDF_{\xi,\eta}}{\normal(m, \Sigma)}
        - \DKL{\PDF_\xi}{\normal(m_\xi, \Sigma_{\xi,\xi})}
        - \DKL{\PDF_\eta}{\normal(m_\eta, \Sigma_{\eta,\eta})}
    \end{multline*}
    Note that
    \[
        \DKL{\PDF_\xi}{\normal(m_\xi, \Sigma_{\xi,\xi})} + \DKL{\PDF_\eta}{\normal(m_\eta, \Sigma_{\eta,\eta})} = \DKL{\PDF_\xi \otimes \PDF_\eta}{\normal(m, \diag(\Sigma_{\xi,\xi}, \Sigma_{\eta,\eta})},
    \]
    which results in
    \begin{multline*}
        I(\xi; \eta) = \frac{1}{2} \left[ \log \det \Sigma_{\xi,\xi} + \log \det \Sigma_{\eta,\eta} - \log \det \Sigma \right] + \\
        + \DKL{\PDF_{\xi,\eta}}{\normal(m, \Sigma)} -  \DKL{\PDF_\xi \otimes \PDF_\eta}{\normal(m, \diag(\Sigma_{\xi,\xi}, \Sigma_{\eta,\eta})}
    \end{multline*}
\end{proofE}

\begin{corollaryE}
    \label{corollary:MI_after_Gaussianization_bounds}
    Under the assumptions of~\cref{corollary:MI_after_Gaussianization},
    \[
        \left| I(\xi;\eta) - \frac{1}{2} \left[ \log \det \Sigma_{\xi,\xi} + \log \det \Sigma_{\eta,\eta} - \log \det \Sigma \right] \right| \leq \DKL{\PDF_{\xi,\eta}}{ \normal(m, \Sigma) }.
    \]
\end{corollaryE}

\begin{proofE}
    The same as for~\cref{corollary:MI_approximation_decomposition_bounds}
\end{proofE}

\begin{remark}
    \label{remark:MI_after_Gaussianization_bounds_tightness}
    The upper bound from~\cref{corollary:MI_after_Gaussianization_bounds} is tight, consider $ \xi \sim \normal(0,1) $, $ \eta = (2 B - 1) \cdot \xi $, where $ B \sim \textnormal{Bernoulli}(1/2) $ and is independent of $ \xi $.
\end{remark}

From now on we denote $ f_X(X) $ and $ f_Y(Y) $ as $ \xi $ and $ \eta $ correspondingly.
Note that, in contrast to~\cref{theorem:MI_approximation_decomposition} and \cref{corollary:MI_approximation_decomposition_bounds}, $ I_q(\xi;\eta) $ is replaced by a closed-form expression, which is not possible to achieve in general.
%(i.e., when $ q $ is not a Gaussian PDF).
The provided closed-form expression allows for calculating MI for jointly Gaussian $ (\xi, \eta) $,
and serves as a lower bound on MI in the general case of $ \xi $ and $ \eta $ being only marginally Gaussian.
%We want to emphasize that the conditions of \cref{theorem:MI_under_nonsingular_mappings} may be violated by Gaussianizing $ (X, Y) $ as a whole due to potential intermixing of components and corresponding information leaks between random vectors.
%In conclusion, to fully leverage the convenient and robust closed-form expression for MI,
%random vectors have to be Gaussianized via separate flows (to preserve the MI),
%but the resulting joint distribution also have to be Gaussian (to achieve equality in~\cref{eq:closed_form_MI_expression_Gaussian}).

\subsection{General binormalization approach}
\label{subsection:general_binormalization_approach}

%Due to the aforementioned nuances, we propose applying flows $ f_X $ and $ f_Y $ to $ X $ and $ Y $ separately.
In order to minimize $ \DKL{\PDF_{\xi,\eta}}{ \normal(m, \Sigma) } $,
we train $ f_X \times f_Y $ as a single normalizing flow.
Instead of maximizing the log-likelihood using two separate and fixed base (latent) distributions,
we maximize the log-likelihood of the joint sampling $ \{ (x_k, y_k) \}_{k=1}^N $ using the whole set of Gaussian distributions
as possible base distributions.
%$ \setnormal = \left\{\normal(m, \Sigma) \mid m \in \reals^{d_\xi + d_\eta}, \Sigma \in \reals^{(d_\xi + d_\eta) \times (d_\xi + d_\eta)} \right\} $ as a set of possible base distributions.

\def\setnormal{S_{\normal}}
\begin{definition}
    \label{definition:normal_distributions_set}
    We denote a set of $ d $-dimensional Gaussian distributions as
    $ \setnormal^d \defeq \left\{\normal(m, \Sigma) \mid m \in \reals^{d}, \Sigma \in \reals^{d \times d} \right\} $.%
    \footnote{We omit $ d $ whenever it can be deduced from the context.}
\end{definition}

%\begin{definition}
%    \label{definition:DKL_for_sets}
%    Let $ \mu $ be a probability distribution,
%    $ S $ be a set of probability distributions.
%    Then
%    \begin{align*}
%        \DKL{\mu}{S} &\defeq \inf_{\nu \in S} \DKL{\mu}{\nu} \\
%        \DKL{S}{\mu} &\defeq \inf_{\nu \in S} \DKL{\nu}{\mu}
%    \end{align*}
%\end{definition}

%Let us denote an empirical distribution by $ \hat \mu $.
%a corresponding sampling by $ \{ \zeta_k \}_{k=1}^N $.
%The value of $ \DKL{\hat \mu}{\setnormal} $ can be expressed in closed-form via %maximum-likelihood estimates for $ \mu $ and $ \Sigma $:

\begin{definition}
    \label{definition:log_likelihood_for_sets}
    The log-likelihood of a sampling $ \{ z_k \}_{k=1}^N $ with respect to a set of absolutely continuous probability distributions is defined as follows:
    \[
        \loglikelihood_S (\{ z_k \}) \defeq \sup_{\mu \in S} \loglikelihood_\mu(\{ z_k \}) = \sup_{\mu \in S} \sum_{k=1}^N \log \left[\left(\frac{d \mu}{d z}\right)(z_k) \right]
    \]
\end{definition}

%In our case, $ \loglikelihood_{\setnormal}\left( \{ (f_X^{-1}(x_k), f_Y^{-1}(y_k)) \} \right) $ can be expressed in a closed-form via the maximum-likelihood estimates for $ m $ and $ \Sigma $.
Let us define $ f \defeq f_X \times f_Y $ (Cartesian product of flows)
and $ S \circ f = \{ \mu \circ f \mid \mu \in S \} $ (set of pushforward measures).
In our case, $ \loglikelihood_{\setnormal \circ f} \left( \{ (x_k, y_k) \} \right) $ can be expressed in a closed-form via the change of variables formula (identically to a classical normalizing flows setup) and maximum-likelihood estimates for $ m $ and $ \Sigma $.

\begin{statementE}
    \label{statement:ML_estimates_maximize_likelihood}

    \def\eqoneleft{\loglikelihood_{\setnormal \circ (f_X \times f_Y)} \left( \{ (x_k, y_k) \} \right)}
    \def\eqoneright{\log \left| \det \frac{\partial f(x,y)}{\partial (x,y)} \right| + \loglikelihood_{\normal(\hat m, \hat \Sigma)} (\{ f(x_k, y_k) \})}

    \def\eqtwoleft{\log \left| \det \frac{\partial f(x,y)}{\partial (x,y)} \right|}
    \def\eqtworight{\log \left| \det \frac{\partial f_X(x)}{\partial x} \right| + 
        \log \left| \det \frac{\partial f_Y(y)}{\partial y} \right|}

    \[
        \loglikelihood_{\setnormal \circ (f_X \times f_Y)} \left( \{ (x_k, y_k) \} \right) = \log \left| \det \frac{\partial f(x,y)}{\partial (x,y)} \right| + \loglikelihood_{\normal(\hat m, \hat \Sigma)} (\{ f(x_k, y_k) \}),
    \]
    where
    \[
        \log \left| \det \frac{\partial f(x,y)}{\partial (x,y)} \right| = \log \left| \det \frac{\partial f_X(x)}{\partial x} \right| + 
        \log \left| \det \frac{\partial f_Y(y)}{\partial y} \right|,
    \]
    %\vspace{-\baselineskip}
    \begin{equation*}
        \hat m = \frac{1}{N} \sum_{k=1}^N f(x_k,y_k), \qquad
        \hat \Sigma = \frac{1}{N} \sum_{k=1}^N (f(x_k,y_k) - \hat m) (f(x_k,y_k) - \hat m)^T
    \end{equation*}
\end{statementE}

\begin{proofE}
    Due to the change of variables formula,
    \[
        \loglikelihood_{\setnormal \circ (f_X \times f_Y)} \left( \{ (x_k, y_k) \} \right) = \log \left| \det \frac{\partial f(x,y)}{\partial (x,y)} \right| + \loglikelihood_{\normal} (\{ f(x_k, y_k) \})
    \]
    As $ f = f_X \times f_Y $, the Jacobian matrix $ \frac{\partial f(x,y)}{\partial(x, y)} $ is block-diagonal, so
    \[
        \log \left| \det \frac{\partial f(x,y)}{\partial (x,y)} \right| =
        \log \left| \det \frac{\partial f_X(x)}{\partial x} \right| + 
        \log \left| \det \frac{\partial f_Y(y)}{\partial y} \right|
    \]
    Finally, we use the maximum-likelihood estimates for $ m $ and $ \Sigma $ to drop the supremum in $ \loglikelihood_{\normal} (\{ f(x_k, y_k) \}) $:
    \begin{multline*}
        \hat m = \frac{1}{N} \sum_{k=1}^N f(x_k,y_k), \quad
        \hat \Sigma = \frac{1}{N} \sum_{k=1}^N (f(x_k,y_k) - \hat m) (f(x_k,y_k) - \hat m)^T
        \quad \Longrightarrow \\
        \Longrightarrow \quad
        \loglikelihood_{\normal} (\{ f(x_k, y_k) \}) = \loglikelihood_{\normal(\hat m, \hat \Sigma)} (\{ f(x_k, y_k) \})
    \end{multline*}
\end{proofE}

Maximization of $ \loglikelihood_{\setnormal \circ (f_X \times f_Y)} \left( \{ (x_k, y_k) \} \right) $ with respect to parameters of $ f_X $ and $ f_Y $
%yields a pair of \emph{binormalizing} flows,
minimizes $ \DKL{\PDF_{\xi,\eta}}{ \normal(m, \Sigma) } $~\cite{kobyzev2021normflows_overview},
making it possible to apply \cref{theorem:MI_under_nonsingular_mappings} and \cref{corollary:MI_after_Gaussianization} to acquire an MI estimate with corresponding non-asymptotic error bounds from~\cref{corollary:MI_after_Gaussianization_bounds}:
\begin{equation}
    \label{eq:general_Gaussian_MI_estimate}
    \hat I_{\normal\textnormal{-MIENF}}(\{ (x_k, y_k) \}_{k=1}^N)%(X; Y)
    \defeq \frac{1}{2} \left[ \log \det \hat \Sigma_{\xi,\xi} + \log \det 
    \hat \Sigma_{\eta,\eta} - \log \det 
    \hat \Sigma \right]
\end{equation}

Note that if only marginal Gaussianization is achieved,
\cref{eq:general_Gaussian_MI_estimate} serves as a lower bound estimate.
As $ \hat I_{\normal\textnormal{-MIENF}} $ involves maximum-likelihood estimates of covariance matrices, existing results can be employed to acquire the asymptotic variance:
\begin{lemmaE}[Lemma 2 in~\cite{duong2023dine}]
    Let $ f_X, f_Y $ be fixed.
    Let $ (\xi,\eta) $ has finite covariance matrix.
    Then, the asymptotic variance of $ \hat I_{\normal\textnormal{-MIENF}} $ is $ O(d^2/N) $, with $ d $ being the dimensionality.
    If $ (\xi,\eta) $ is also Gaussian, the asymptotic variance is further improved to $ O(d/N) $.
\end{lemmaE}

\begin{proofE}
    The variance of $ \hat I_{\normal\textnormal{-MIENF}} $ is upper bounded by the sum of the variances of the log-det terms.
    If $ (\xi,\eta) $ is Gaussian, the log-det terms are asymptotically normal with the asymptotic variance being $ 2d/N $ (Corollary 1 in~\cite{cai2015logdet}).
    If $ (\xi, \eta) $ is not Gaussian, the central limit theorem can be applied to each element of the covariance matrix,
    which in combination with the delta method yields the final result.
\end{proofE}

\subsection{Refined approach}
\label{subsection:refined_approach}

Although the proposed general method is compelling,
as it requires only the slightest modifications to the conventional normalizing flows setup to make the application of the closed-form expressions for MI possible,
we have to mention several drawbacks.

Firstly, the log-likelihood maximum is ambiguous,
as $ \loglikelihood_{\setnormal} $ is invariant under invertible affine mappings,
which makes the proposed log-likelihood maximization an ill-posed problem:
\begin{remarkE}
    Let $ A \in \reals^{d \times d} $ be a non-singular matrix,
    $ b \in \reals $,
    $ \{ z_k \}_{k = 1}^N \subseteq \reals^d $.
    Then
    \[
        \loglikelihood_{\setnormal \circ (A z + b)} (\{z_k\}) = \loglikelihood_{\setnormal} (\{z_k\})
    \]
\end{remarkE}

\begin{proofE}
    \begin{multline*}
        \loglikelihood_{\setnormal \circ (A z + b)} (\{z_k\})
        = \log |\det A | + \loglikelihood_{\setnormal} (\{A z_k + b\})
        = \log |\det A | + \loglikelihood_{\normal(\hat m + b, A \hat \Sigma A^T)} (\{A z_k + b\}) = \\
        = \log |\det A | + \log |\det A^{-1} | + \loglikelihood_{\normal(\hat m, \hat \Sigma)} (\{z_k\})
        = \loglikelihood_{\setnormal} (\{z_k\})
        %\loglikelihood_{\setnormal \circ (A z + b)} (\{z_k\}) = \\
        %\begin{aligned}
        %&= \log |\det A | + \loglikelihood_{\setnormal} (\{A z_k + b\}) = \\
        %&= \log |\det A | + \loglikelihood_{\normal(\hat m + b, A \hat \Sigma A^T)} (\{A z_k + b\}) = \\
        %&= \log |\det A | + \log |\det A^{-1} | + \loglikelihood_{\normal(\hat m, \hat \Sigma)} (\{z_k\}) =
        %\end{aligned}
        %\\ = \loglikelihood_{\setnormal} (\{z_k\})
    \end{multline*}
\end{proofE}

%This makes the proposed log-likelihood maximization an ill-posed problem.

Secondly, this method requires a regular (ideally, after every gradient descend step) updating of $ \hat m $ and $ \hat \Sigma $ for the whole dataset, which is expensive.
In practice, these estimates can be replaced with batchwise maximum likelihood estimates,
which are used to update $ \hat m $ and $ \hat \Sigma $ via exponential moving average (EMA).
This approach, however, requires tuning EMA multiplication coefficient in accordance with the learning rate to make the training fast yet stable.
We also note that $ \hat m $ and $ \hat \Sigma $ can be made learnable via the gradient ascend,
but the benefits of the closed-form expressions for $ \loglikelihood_{\setnormal} $ in \cref{statement:ML_estimates_maximize_likelihood} are thus lost.

Finally, each loss function evaluation requires inversion of $ \hat \Sigma $,
and each MI estimation requires evaluation of $ \det \hat \Sigma $ and determinants of two diagonal blocks of $ \hat \Sigma $.
This might be resource-consuming in high-dimensional cases,
as matrices may not be sparse.
Numerical instabilities might also occur if $ \hat \Sigma $ happens to be ill-posed (might happen in the case of data lying on a manifold or due to high MI).

That is why we propose an elegant and simple way to eliminate all the mentioned problems by further narrowing down $ \setnormal $ to a subclass of Gaussian distributions with simple and bounded covariance matrices and fixed means.
This approach is somewhat reminiscent of the non-linear canonical correlation analysis~\cite{lancaster1958bivariate_structure, hannan1961cca}.
%For simplicity, we assume $ d_\xi \leq d_\eta $ from now on.

\def\settridiagnormal{S_{\textnormal{tridiag-} \normal}}
\begin{definition}
    \label{definition:tridiagonal_normal_distributions_set}
    \begin{equation*}
        \settridiagnormal^{d_\xi, d_\eta} \defeq \left\{\normal(0, \Sigma) \mid \Sigma_{\xi,\xi} = I_{d_\xi}, \Sigma_{\eta,\eta} = I_{d_\eta},
        \Sigma_{\xi,\eta} (\equiv \Sigma_{\eta,\xi}^T) = \diag (\{ \rho_j \}), \rho_j \in [0;1) \right\}
    \end{equation*}
\end{definition}

This approach solves all the aforementioned problems without any loss in generalization ability, as it is shown by the following results:

\begin{corollaryE}
    \label{corollary:simplified_closed_form_MI_expression_Gaussian}
    If $ (\xi, \eta) \sim \mu \in \settridiagnormal $, then
    \begin{equation}
        \label{eq:simplified_closed_form_MI_expression_Gaussian}
        I(\xi; \eta) = - \frac{1}{2} \sum_j \log (1 - \rho_j^2)
    \end{equation}
    %In case of $ \rho_j \equiv \rho $,
    %this expression can be further simplified:
    %\begin{equation}
    %    \label{eq:constant_simplified_closed_form_MI_expression_Gaussian}
    %    I(\xi; \eta) = - \frac{d'}{2} \log (1 - \rho^2), \quad d' \defeq \min\{ d_\xi, d_\eta\}
    %\end{equation}
\end{corollaryE}

\begin{proofE}
    Under the proposed assumptions, $ \log \det \Sigma_{\xi,\xi} = \log \det \Sigma_{\eta,\eta} = 0 $,
    so $ I(\xi; \eta) = -\frac{1}{2} \log \det \Sigma $.
    The matrix $ \Sigma $ is block-diagonalizable via the following permutation:
    \[
        (\xi_1, \ldots, \xi_{d_\xi}, \eta_1, \ldots, \eta_{d_\eta}) \mapsto (\xi_1, \eta_1, \xi_2, \eta_2, \ldots),
    \]
    with the blocks being
    \[
        \Sigma_{\xi_i, \eta_i} =
        \begin{bmatrix}
            1 & \rho_j \\
            \rho_j & 1
        \end{bmatrix}
    \]
    The determinant of block-diagonal matrix is a product of the block determinants, so $ I(\xi; \eta) = - \frac{1}{2} \sum_k \log (1 - \rho_j^2) $.
\end{proofE}

%\def\settridiagnormal{S_{\textnormal{tridiag} \, \normal}}
%We define a subset $ \settridiagnormal^{d_\xi, d_\eta} \subseteq \setnormal $ as a set of normal distributions with covariance matrices,
%satisfying the conditions of \cref{corollary:simplified_closed_form_MI_expression_Gaussian},
%and zero means:
%
%\begin{definition}
%    \label{definition:tridiagonal_normal_distributions_set}
%    \begin{multline*}
%        \setnormal^{d_\xi + d_\eta} \supseteq
%        \settridiagnormal^{d_\xi, d_\eta} \defeq \left\{\normal(0, \Sigma) \mid (\Sigma)_{i,j} = \right. \\
%        \left. \delta_{i,j} + \rho_j (\delta_{i-d_\xi,j} + \delta_{i,j-d_\xi}), \rho_j \in [0;1) \right\}%
%    \end{multline*}
%\end{definition}

%This approach solves all the aforementioned problems without any loss in generalization ability, as it is shown in the following Statements:

\begin{statementE}[Canonical correlation analysis]
    \label{statement:MI_preserving_tridiagonalization}
    Let $ (\xi, \eta) \sim \normal(m, \Sigma) $,
    where $ \Sigma $ is non-singular.
    There exist invertible affine mappings $ \varphi_\xi $, $ \varphi_\eta $
    such that $ (\varphi_\xi(\xi), \varphi_\eta(\eta)) \sim \mu \in \settridiagnormal $.
    Due to \cref{theorem:MI_under_nonsingular_mappings}, the following also holds:
    $ I(\xi; \eta) = I(\varphi_\xi(\xi); \varphi_\eta(\eta)) $.
\end{statementE}

\begin{proofE}
    Let us consider centered $ \xi $ and $ \eta $,
    as shifting is an invertible affine mapping.
    Note that $ \Sigma_{\xi, \xi} $ and $ \Sigma_{\eta, \eta} $ are positive definite and symmetric,
    so the following symmetric matrix square roots are defined: $ A = \Sigma_{\xi, \xi}^{-1/2} $, $ B = \Sigma_{\eta, \eta}^{-1/2} $.
    By applying these invertible linear transformation to $ \xi $ and $ \eta $ we get
    \[%begin{multline*}
        \cov(A \xi, B \eta) = %\\ =
        \begin{bmatrix}
            \Sigma_{\xi, \xi}^{-1/2} \Sigma_{\xi,\xi} (\Sigma_{\xi, \xi}^{-1/2})^T & \Sigma_{\xi, \xi}^{-1/2} \Sigma_{\xi, \eta} (\Sigma_{\eta, \eta}^{-1/2})^T \\
            \Sigma_{\eta, \eta}^{-1/2} \Sigma_{\eta, \xi} (\Sigma_{\xi, \xi}^{-1/2})^T & \Sigma_{\eta, \eta}^{-1/2} \Sigma_{\eta,\eta} (\Sigma_{\eta, \eta}^{-1/2})^T
        \end{bmatrix}
        = %\\ =
        \begin{bmatrix}
            I & C \\
            C^T & I
        \end{bmatrix},
    \]%end{multline*}
    where $ C =  \Sigma_{\xi, \xi}^{-1/2} \Sigma_{\xi, \eta} (\Sigma_{\eta, \eta}^{-1/2})^T $.
    
    Then, the singular value decomposition is performed: $ C = U R V^T $,
    where $ U $ and $ V $ are orthogonal, $ R = \diag(\{ \rho_j \}) $.
    Finally, we apply $ U^T $ to $ A\xi $ and $ V^T $ to $ B \eta $:
    \[%begin{multline*}
        \cov(U^T A \xi, V^T B \eta) = %\\ =
        \begin{bmatrix}
            U^T U & U^T C V \\
            (U^T C V)^T & V^T V
        \end{bmatrix}
        =
        \begin{bmatrix}
            I & R \\
            R^T & I
        \end{bmatrix},
    \]%end{multline*}
    Note that $ U^T A $ and $ V^T B $ are invertible.
\end{proofE}

\begin{statementE}
    \label{statement:tridiagonal_to_diagonal_Gaussian}
    % TODO: разобраться с неравными размерностями в нормализующей матрице.
    Let $ (\xi, \eta) \sim \normal(0, \Sigma) \in \settridiagnormal $, $ \{ z_k \}_{k=1}^N \subseteq \reals^{d_\xi + d_\eta} $.
    Then
    \[
        \loglikelihood_{\normal(0, \Sigma)}(\{ z_k \} ) = I(\xi; \eta) + \loglikelihood_{\normal(0, I)}(\{ \Sigma^{-1/2} z_k \}),
    \]
    where
    \[
        \Sigma^{-1/2} =
        %\left[
        %    \begin{array}{c|c}
        %        \diag(\alpha_j + \beta_j) & \diag(\alpha_j - \beta_j) \\
        %        \hline
        %        \undermat{d_\xi}{\diag(\alpha_j - \beta_j)} & \undermat{d_\eta}{\diag(\alpha_j + \beta_j)} \\
        %    \end{array}
        %\right]
        %\vspace{16pt}
        \left[
            \begin{array}{cc|cc}
                \alpha_j + \beta_j &        & \alpha_j - \beta_j & \\
                                   & \ddots &                    & \ddots \\
                \hline
                \alpha_j - \beta_j &        & \alpha_j + \beta_j & \\
                \undermat{d_\xi}{ \hphantom{\alpha_j + \beta_j} & \ddots} & \undermat{d_\eta}{ \hphantom{\alpha_j + \beta_j} & \ddots} \\
            \end{array}
        \right]
        \vspace{16pt}
        %\left[
        %    \begin{array}{ccc|ccc}
        %        \ddots &                    &        & \ddots &                    & \\
        %               & \alpha_j + \beta_j &        &        & \alpha_j - \beta_j & \\
        %               &                    & \ddots &        &                    & \ddots \\
        %        \hline
        %        \ddots &                    &        & \ddots &                    & \\
        %               & \alpha_j - \beta_j &        &        & \alpha_j + \beta_j & \\
        %        \undermat{d_\xi}{ \hphantom{\ddots} & \hphantom{\alpha_1 + \beta_1} & \ddots} & \undermat{d_\eta}{ \hphantom{\ddots} & \hphantom{\alpha_1 + \beta_1} & \ddots} \\
        %    \end{array}
        %\right]
        \qquad
        \begin{aligned}
        \alpha_j &= \frac{1}{2 \sqrt{1 + \rho_j}} \\
        \beta_j  &= \frac{1}{2 \sqrt{1 - \rho_j}}
        \end{aligned}
    \]
    
    \vspace{1.0\baselineskip}
    and $ I(\xi;\eta) $ is calculated via~\eqref{eq:simplified_closed_form_MI_expression_Gaussian}.
\end{statementE}

\begin{proofE}
    Note that $ \Sigma $ is positive definite and symmetric,
    so the following symmetric matrix square root is defined: $ \Sigma^{-1/2} $.
    This matrix is a normalization matrix:
    $ \cov(\Sigma^{-1/2} (\xi,\eta)) = \Sigma^{-1/2} \, \Sigma \, (\Sigma^{-1/2})^T = I $.
    
    According to the change of variable formula,
    \[
        \loglikelihood_{\normal(0, \Sigma)}(\{ z_k \} ) = \log\det \Sigma^{-1/2} + \loglikelihood_{\normal(0, I)}(\{ \Sigma^{-1/2} z_k \})
    \]
    As $ \Sigma_{\xi,\xi} = I_{d_\xi} $ and $ \Sigma_{\eta,\eta} = I_{d_\eta} $,
    from the equation~\eqref{eq:closed_form_MI_expression_Gaussian} we derive
    \[
        I(\xi;\eta) = - \frac{1}{2} \log \det \Sigma = \log \det \Sigma^{-1/2}
    \]

    Finally, it is sufficient to validate the proposed closed-form expression for $ \Sigma^{1/2} $ in the case of $ 2 \times 2 $ matrices,
    as $ \Sigma $ is block-diagonalizable (with $ 2 \times 2 $ blocks) via the following permutation:
    \[
        M \colon (\xi_1, \ldots, \xi_{d_\xi}, \eta_1, \ldots, \eta_{d_\eta}) \mapsto (\xi_1, \eta_1, \xi_2, \eta_2, \ldots),
    \]
    %The validation:
    %\[%begin{multline*}
    %    \begin{bmatrix}
    %        \alpha + \beta & \alpha - \beta \\
    %        \alpha - \beta & \alpha + \beta
    %    \end{bmatrix}
    %    \cdot
    %    \begin{bmatrix}
    %        1 & \rho \\
    %        \rho & 1
    %    \end{bmatrix}
    %    \cdot
    %    \begin{bmatrix}
    %        \alpha + \beta & \alpha - \beta \\
    %        \alpha - \beta & \alpha + \beta
    %    \end{bmatrix} %= \\
    %    =
    %    \begin{bmatrix}
    %        \alpha^2 (1 + \rho) + \beta^2 (1 - \rho) & \alpha^2 (1 + \rho) - \beta^2 (1 - \rho) \\
    %        \alpha^2 (1 + \rho) - \beta^2 (1 - \rho) & \alpha^2 (1 + \rho) + \beta^2 (1 - \rho)
    %    \end{bmatrix} %= \\
    %    =
    %    \begin{bmatrix}
    %        1 & 0 \\
    %        0 & 1
    %    \end{bmatrix}
    %\]%end{multline*}
    Note that
    \[
        \begin{bmatrix}
            \alpha + \beta & \alpha - \beta \\
            \alpha - \beta & \alpha + \beta
        \end{bmatrix}^2
        =
        2 \cdot
        \begin{bmatrix}
            \alpha^2 + \beta^2 & \alpha^2 - \beta^2 \\
            \alpha^2 - \beta^2 & \alpha^2 + \beta^2
        \end{bmatrix}
        =
        \frac{1}{1 - \rho^2}
        \begin{bmatrix}
            1 & -\rho \\
            -\rho & 1
        \end{bmatrix}
    \]
    \[
        \frac{1}{1 - \rho^2}
        \begin{bmatrix}
            1 & -\rho \\
            -\rho & 1
        \end{bmatrix}
        \cdot
        \begin{bmatrix}
            1 & \rho \\
            \rho & 1
        \end{bmatrix}
        =
        \begin{bmatrix}
            1 & 0 \\
            0 & 1
        \end{bmatrix}
    \]
\end{proofE}

Maximization of $ \loglikelihood_{\settridiagnormal \circ (f_X \times f_Y)} \left( \{ (x_k, y_k) \} \right) $ with respect to the parameters of $ f_X $ and $ f_Y $ yields the following MI estimate:
\begin{equation}
    \label{eq:refined_Gaussian_MI_estimate}
    \hat I_{\textnormal{tridiag-} \normal \textnormal{-MIENF}}(\{ (x_k, y_k) \}_{k=1}^N)%(X; Y)
    \defeq -\frac{1}{2} \sum_j \log (1 - \hat \rho_j^2),
\end{equation}
where $ \hat \rho_j $ are the maximum-likelihood estimates of $ \rho_j $.

\subsection{Tractable error bounds}
\label{subsection:tractable_bounds}

Note that~\cref{corollary:MI_approximation_decomposition_bounds} and \cref{corollary:MI_after_Gaussianization_bounds} provide us with non-asymptotic, but untractable bounds.
These bounds can be estimated via various KL divergence estimators~\cite{donsker1983dv, nguyen2007nwj, belghazi2018mine, nishiyama2019new_lower_bound_on_kld, franzese2024minde}.
However, this requires training an additional neural network, which is computationally expensive.

Conveniently, as the proposed method involves maximization of the likelihood, a cheap and tractable lower bound on the KL divergence can be obtained via an entropy-cross-entropy decomposition:
\begin{equation}
    \label{eq:entropy_cross_entropy_decomposition}
    \DKL{\PDF}{q} = \expect_{Z \sim \PDF} \log \frac{p(Z)}{q(Z)} = -\expect_{Z \sim \PDF} \log q(Z) - h(Z) \geq -\expect_{Z \sim \PDF} \log q(Z) - h(\normal(m, \Sigma))
\end{equation}
Note that $ \expect \log q(Z) $ in~\cref{eq:entropy_cross_entropy_decomposition} is estimated by the log-likelihood of the samples in the latent space (which is inevitably evaluated each training step),
and $ h(Z) $ can be upper bounded by the entropy of a Gaussian distribution (see~\cref{theorem:Gaussian_maximizes_entropy}).
%Combining all these and considering a Gaussian $ q $, we arrive at the following result:
Unfortunately, as~\cref{theorem:Gaussian_maximizes_entropy} has already been employed to derive~\cref{corollary:MI_after_Gaussianization}, the proposed bound is trivial (equaling to zero) in our Gaussian-based setup.
However, this idea might still be useful in the general case.

%\begin{lemmaE}
%    In the setting of~\cref{corollary:MI_after_Gaussianization_bounds},
%    $ \DKL{\PDF_{\xi,\eta}}{ \normal(m, \Sigma) } \geq  $
%\end{lemmaE}

\subsection{Implementation details}
\label{subsection:implementation_details}

In this section, we would like to emphasize the computational simplicity of the proposed amendments to the conventional normalizing flows setup.

Firstly, \cref{statement:tridiagonal_to_diagonal_Gaussian} allows for a cheap log-likelihood computation and sample generation,
as $ \Sigma $, $ \Sigma^{-1/2} $ and $ \Sigma^{-1} $ are easily calculated from the $ \{ \rho_j \} $ and are sparse (tridiagonal, block-diagonalizable with $ 2 \times 2 $ blocks).
Secondly, the method requires only $ d' = \min \{d_\xi, d_\eta \} $ additional parameters: the estimates for $ \{ \rho_j\} $.
As $ \rho_j \in [0;1) $, an appropriate parametrization should be chosen to allow for stable learning of $ \{ \hat \rho_j \} $ via the gradient ascend.
We propose using the \emph{logarithm of cross-component MI}:
$ w_j \defeq \log I (\xi_j; \eta_j) = \log\left[ -\frac{1}{2} \log (1 - \rho_j^2)  \right] $.
In this parametrization $ w_j \in \reals $ and
\[
    \hat I_{\textnormal{tridiag-} \normal \textnormal{-MIENF}}(\{ (x_k, y_k) \}_{k=1}^N) = \sum_j e^{\hat w_j}, \quad
    \rho_j = \sqrt{1 - \exp(-2 \cdot e^{w_j})} \in (0;1)
\]

Although $ \rho_j = 0 $ is not achievable in the proposed parametrization, it can be made arbitrary close to $ 0 $ 
with $ w_j \to -\infty $.

%\subsection{Convergence}
%\label{subsection:convergence}
%
%Convergence of our estimator directly follows from the convergence in distribution.
%
%\begin{theoremE}
%    \label{theorem:MI_convergence_from_continuity}
%    If $ I(\xi;\eta) $ i
%\end{theoremE}

\subsection{Extension to non-Gaussian base distributions and non-bijective flows}
%\label{subsection:extension_to_non_Gaussian_base_distributions}

The proposed method can be easily generalized to account for any base distribution with closed-form expression for MI, or even a combination of such distributions.
For example, a smoothed uniform distribution can be considered,
with the learnable parameter being the smoothing constant $ \varepsilon $,
see~\cref{subsection:smoothed_uniform_distribution}, \cref{eq:smoothed_uniform_mutual_information}.
However, due to~\cref{remark:no_DKL_entropy_formula_in_general},
neither~\cref{corollary:MI_approximation_decomposition_bounds}, nor~\cref{corollary:MI_after_Gaussianization_bounds} can be used to bound the estimation error in this case.
%in case of non-Gaussian base distributions acquired estimates may be neither upper nor lower bounds on MI.

%\subsection{Extension to non-bijective flows}
%\label{subsection:extension_to_non_bijective_flows}

Also note that, as~\cref{theorem:MI_under_nonsingular_mappings} does not require $ g $ to be bijective,
our method is naturally extended to injective normalizing flows~\cite{teng2019invertible_AE, brehmer2020flows_manifold_learning}.
Moreover, according to~\cite{butakov2024lossy_compression},
such approach may actually facilitate the estimation of MI.

\section{Experiments}
\label{section:experiments}

To evaluate our estimator,
we utilize synthetic datasets with known mutual information.
In~\cite{czyz2023beyond_normal} and \cite{butakov2024lossy_compression}, extensive frameworks for evaluation of MI estimators have been proposed.
In our work, we borrow complex high-dimensional tests from~\cite{butakov2024lossy_compression} and all non-Gaussian base distributions with known MI from~\cite{czyz2023beyond_normal} (see~\cref{appendix:non_gaussian_based_tests} for more details).
The formal description of the dataset generation and estimator evaluation is provided in~\cref{algorithm:dataset_generation_and_estimator_evaluation}.
Essentially similar setups are widely used to test the MI estimators~\cite{kraskov2004KSG, belghazi2018mine, poole2019on_variational_bounds_MI, mcallester2020limitations_MI, butakov2024lossy_compression, franzese2024minde}.
%Note that any absolutely continuous random vector can be \emph{separately} modeled via the proposed way;
%the theoretical construction for invertible Gaussianization is provided in the work of~\citet{chen2000gaussianization}.
%However, \citet{czyz2023beyond_normal} have suggested that not every \emph{pair} of random vectors can be modeled via the described algorithm,
%as the base distribution may be non-Gaussian.

%Unfortunately, little is known about multidimensional non-Gaussian distributions with closed-form expression for mutual information and easy sampling procedure.
%To our knowledge, we can only mention the multivariate Student distribution~\cite{arellano_valle2013MI_for_skew_distributions} and the smoothed uniform distribution~\cite{czyz2023beyond_normal}, see~\cref{appendix:non_gaussian_based_tests} for more details.
%That is why we also evaluate our method via the aforementioned distributions,
%see~\cref{figure:multidimensional_tests}.

\begin{algorithm}[ht!]
	\caption{MI estimator evaluation}
	\begin{algorithmic}[1]
    	\STATE Generate two datasets of samples from random vectors $ \xi $ and $ \eta $ with known ground truth mutual information (see~\cref{corollary:MI_after_Gaussianization}, \cref{corollary:simplified_closed_form_MI_expression_Gaussian} and \cref{appendix:non_gaussian_based_tests} for examples): $ \{ (a_k, b_k) \}_{k=1}^N $.
    	\STATE Choose functions $ g_\xi $ and $ g_\eta $ satisfying conditions of \cref{theorem:MI_under_nonsingular_mappings}, so $ I(\xi; \eta) = I\bigl(g_\xi(\xi); g_\eta(\eta)\bigr) $.
        %\STATE Obtain a dataset for $ \bigl( g_\xi(\xi), g_\eta(\eta) \bigr) $: $ \{ (g_\xi(a_k), g_\eta(b_k)) \}_{k=1}^N $. 
    	\STATE Estimate $ I \bigl(g_\xi(\xi); g_\eta(\eta) \bigr) $ via $ \{ (g_\xi(a_k), g_\eta(b_k)) \}_{k=1}^N $ and compare the estimated value with the ground truth.
	\end{algorithmic}
	\label{algorithm:dataset_generation_and_estimator_evaluation}
\end{algorithm}

For the first set of experiments, we map a low-dimensional correlated Gaussian distribution to a distribution of high-dimensional images of geometric shapes.
We compare our method with the Mutual Information Neural Estimator (MINE)~\cite{belghazi2018mine},
Nguyen-Wainwright-Jordan (NWJ)~\cite{nguyen2007nwj, belghazi2018mine} and Nishiyama's~\cite{nishiyama2019new_lower_bound_on_kld} estimators,
nearest neighbor Kraskov-Stoegbauer-Grassberger~\cite{kraskov2004KSG} and $ 5 $-nearest neighbors weighted Kozachenko-Leonenko (WKL) estimator~\cite{kozachenko1987entropy_of_random_vector, berrett2019efficient_knn_entropy_estimation};
the latter is fed with the data compressed via a convolutional autoencoder (CNN AE) in accordance to the pipeline from~\cite{butakov2024lossy_compression}.
%The results are provided in~\cref{figure:compare_methods_images}
%We choose AE+WKL, as this pipeline has performed the best in~\cite{butakov2024lossy_compression} during the same tests we utilize here.

For the second set of experiments, incompressible, high-dimensional non-Gaussian-based distributions are considered. These experiments are conducted to evaluate the robustness of our estimator to the distributions,
which can not be precisely Gaussianized via a Cartesian product of two flows. Here we compare our method to the ground truth only, as similar experiments have been conducted to test other estimators in~\cite{czyz2023beyond_normal}.
%The results are provided in~\cref{figure:multidimensional_tests}.

\begin{figure}[t!]
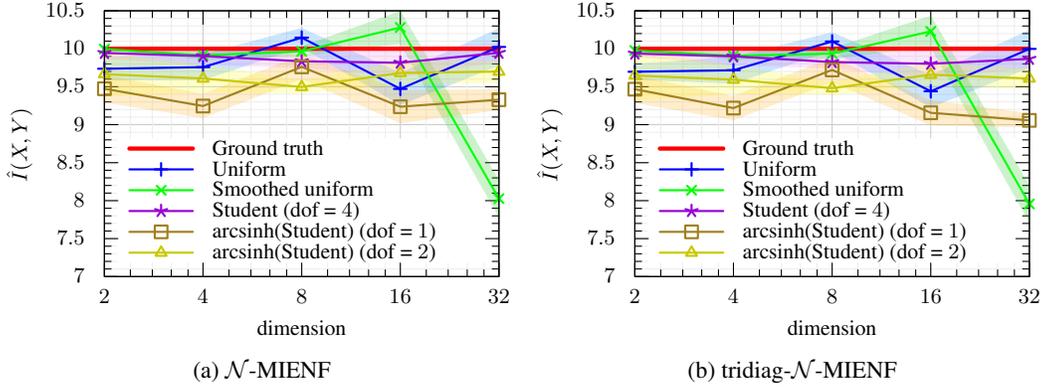

    %\vskip 0.2in
    \begin{subfigure}[t]{0.5\textwidth}
        \centering
        %\fontsize{8}{12}\selectfont
        \small
        \begin{gnuplot}[terminal=tikz, terminaloptions={color size 7.1cm,4.7cm fontscale 0.9}]
            load "./gnuplot/common.gp"
            
            MI = 10.0
            uniform_file_path = "./data/N-MIENF/uniform.csv"
            smoothed_uniform_file_path = "./data/N-MIENF/smoothed_uniform.csv"
            Student_dof_4_file = "./data/N-MIENF/Student_dof_4.csv
            arcsinh_dof_1_file = "./data/N-MIENF/arcsinh_Student_dof_1.csv"
            arcsinh_dof_2_file = "./data/N-MIENF/arcsinh_Student_dof_2.csv"
            confidence = 0.999

            set yrange [7:*]
    
            load "./gnuplot/multidim.gp"
        \end{gnuplot}
        \vspace{-1.0\baselineskip}
        \caption{$ \normal $-MIENF}
    \end{subfigure}
    \begin{subfigure}[t]{0.5\textwidth}
        \centering
        %\fontsize{8}{12}\selectfont
        \small
        \begin{gnuplot}[terminal=tikz, terminaloptions={color size 7.1cm,4.7cm fontscale 0.9}]
            load "./gnuplot/common.gp"
            
            MI = 10.0
            uniform_file_path = "./data/tridiag-N-MIENF/uniform.csv"
            smoothed_uniform_file_path = "./data/tridiag-N-MIENF/smoothed_uniform.csv"
            Student_dof_4_file = "./data/tridiag-N-MIENF/Student_dof_4.csv
            arcsinh_dof_1_file = "./data/tridiag-N-MIENF/arcsinh_Student_dof_1.csv"
            arcsinh_dof_2_file = "./data/tridiag-N-MIENF/arcsinh_Student_dof_2.csv"
            confidence = 0.999

            set yrange [7:*]
    
            load "./gnuplot/multidim.gp"
        \end{gnuplot}
        \vspace{-1.0\baselineskip}
        \caption{tridiag-$ \normal $-MIENF}
    \end{subfigure}
    \caption{Tests with incompressible multidimensional data. ``Uniform'' denotes the uniformly distributed samples acquired from the correlated Gaussians via the Gaussian CDF. ``Smoothed uniform'' and ``Student'' denote the non-Gaussian-based distributions described in~\cref{appendix:non_gaussian_based_tests}.
    ``arcsinh(Student)'' denotes the $ \arcsinh $ function applied to the ``Student'' example (this is done to avoid numerical instabilities in the case of long-tailed distributions).
    We run each test $ 5 $ times and plot $ 99.9 \% $ asymptotic Gaussian CIs. $ 10 \cdot 10^3 $ samples were used. Note that $ \normal $-MIENF and tridiag-$ \normal $-MIENF yield almost the same results with similar bias.}
    \label{figure:multidimensional_tests}
    %\vskip -0.2in
\end{figure}

\begin{figure*}[tb!]
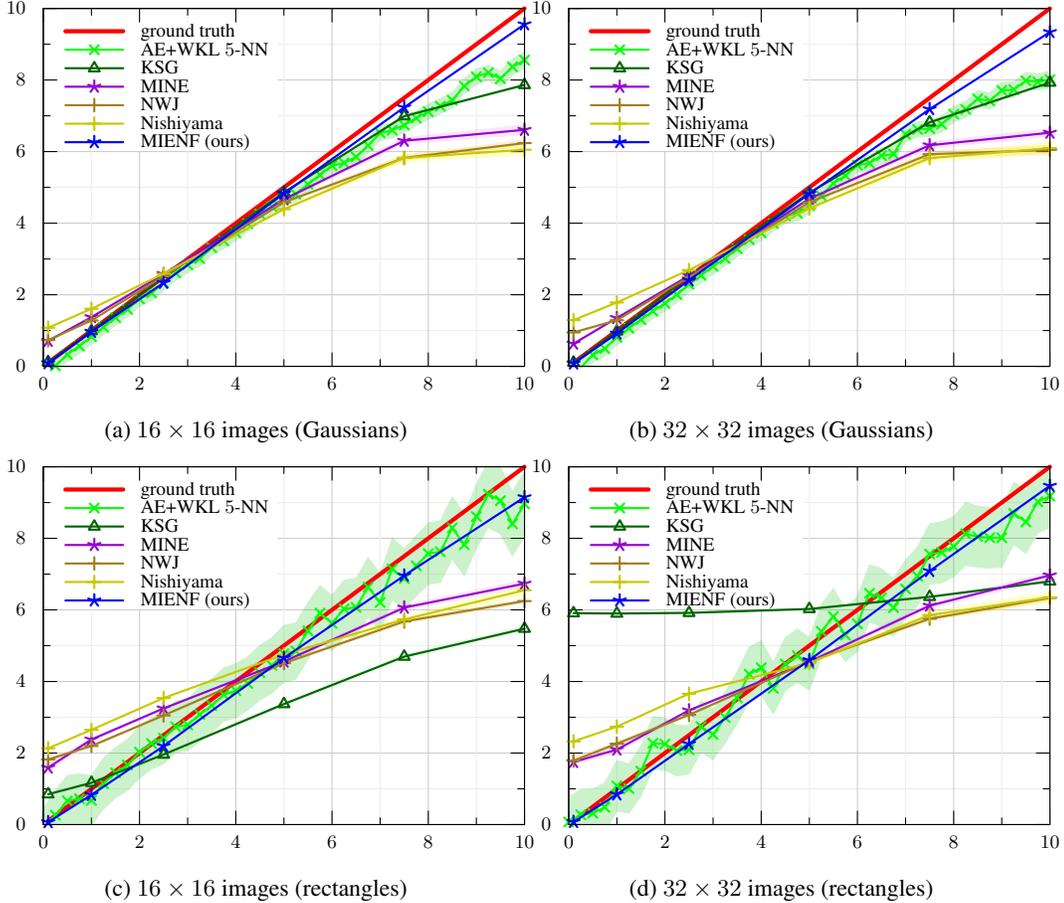

    \centering
    \small
    \begin{subfigure}[t]{0.5\textwidth}
        \centering
        \begin{gnuplot}[terminal=tikz, terminaloptions={color size 7.5cm,5.5cm fontscale 0.8}]
            load "./gnuplot/common.gp"
    
            KL_file_path = "./data/comparison/KL16g.csv"
            KSG_file_path = "./data/comparison/KSG16g.csv"
            NF_file_path = "./data/comparison/NF16g.csv"
            MINE_file_path = "./data/comparison/MINE16g.csv"
            NWJ_file_path = "./data/comparison/NWJ16g.csv"
            Nishiyama_file_path = "./data/comparison/Nishiyama16g.csv"
            
            confidence = 0.999
            max_y = 10.0
    
            load "./gnuplot/estmi_mi.gp"
        \end{gnuplot}
        \vspace{-1.0\baselineskip}
        \caption{$16\times 16$ images (Gaussians)}
    \end{subfigure}%
    \begin{subfigure}[t]{0.5\textwidth}
        \centering
        \begin{gnuplot}[terminal=tikz, terminaloptions={color size 7.5cm,5.5cm fontscale 0.8}]
            load "./gnuplot/common.gp"
    
            KL_file_path = "./data/comparison/KL32g.csv"
            KSG_file_path = "./data/comparison/KSG32g.csv"
            NF_file_path = "./data/comparison/NF32g.csv"
            MINE_file_path = "./data/comparison/MINE32g.csv"
            NWJ_file_path = "./data/comparison/NWJ32g.csv"
            Nishiyama_file_path = "./data/comparison/Nishiyama32g.csv"
            
            confidence = 0.999
            max_y = 10.0
    
            load "./gnuplot/estmi_mi.gp"
        \end{gnuplot}
        \vspace{-1.0\baselineskip}
        \caption{$32\times 32$ images (Gaussians)}
    \end{subfigure}

    \begin{subfigure}[t]{0.5\textwidth}
        \centering
        \begin{gnuplot}[terminal=tikz, terminaloptions={color size 7.5cm,5.5cm fontscale 0.8}]
            load "./gnuplot/common.gp"
    
            KL_file_path = "./data/comparison/KL16r.csv"
            KSG_file_path = "./data/comparison/KSG16r.csv"
            NF_file_path = "./data/comparison/NF16r.csv"
            MINE_file_path = "./data/comparison/MINE16r.csv"
            NWJ_file_path = "./data/comparison/NWJ16r.csv"
            Nishiyama_file_path = "./data/comparison/Nishiyama16r.csv"
            
            confidence = 0.999
            max_y = 10.0
    
            load "./gnuplot/estmi_mi.gp"
        \end{gnuplot}
        \vspace{-1.0\baselineskip}
        \caption{$16\times 16$ images (rectangles)}
    \end{subfigure}%
    \begin{subfigure}[t]{0.5\textwidth}
        \centering
        \begin{gnuplot}[terminal=tikz, terminaloptions={color size 7.5cm,5.5cm fontscale 0.8}]
            load "./gnuplot/common.gp"
    
            KL_file_path = "./data/comparison/KL32r.csv"
            KSG_file_path = "./data/comparison/KSG32r.csv"
            NF_file_path = "./data/comparison/NF32r.csv"
            MINE_file_path = "./data/comparison/MINE32r.csv"
            NWJ_file_path = "./data/comparison/NWJ32r.csv"
            Nishiyama_file_path = "./data/comparison/Nishiyama32r.csv"
            
            confidence = 0.999
            max_y = 10.0
    
            load "./gnuplot/estmi_mi.gp"
        \end{gnuplot}
        \vspace{-1.0\baselineskip}
        \caption{$32\times 32$ images (rectangles)}
    \end{subfigure}
    \caption{Comparison of the selected estimators. Along $ x $ axes is $ I(X;Y) $, along $ y $ axes is $ \hat I(X;Y) $. We plot 99.9\% asymptotic confidence intervals acquired either from the MC integration standard deviation (WKL, KSG) or from the epochwise averaging (other methods, $ 200 $ last epochs). $ 10 \cdot 10^3 $ samples were used.}
    \label{figure:compare_methods_images}
    %\vskip 0.2in
\end{figure*}

%We choose MINE as a suitable and well-performing representative of advanced NN-based MI estimators; see the overviews provided by~\citet{poole2019on_variational_bounds_MI, mcallester2020limitations_MI, song2020understanding_limitations, czyz2023beyond_normal}.
%We choose AE+WKL, as this pipeline has performed the best in~\cite{butakov2024lossy_compression} during the same tests we utilize here.

For the tests with synthetic images, we use GLOW~\cite{kingma2018GLOW} normalizing flow architecture with $ \log_2(\textnormal{image size}) $ splits, $ 2 $ blocks between splits and $ 16 $ hidden channels in each block, appended with a learnable orthogonal linear layer and a small $ 4 $-layer Real NVP flow~\cite{dinh2017real_NVP}.
For the other tests, we use $ 6 $-layer Real NVP architecture.
For further details (including the architecture of MINE critic network and CNN autoencoder), we refer the reader to~\cref{appendix:technical_details}.

The results of the experiments performed with the high-dimensional synthetic images are provided in~\cref{figure:compare_methods_images}. 
We attribute the good performance of AE+WKL to the fact that the proposed synthetic datasets are easily and almost losslessly compressed via a CNN AE.
We run additional experiments with much simpler, but incompressible data to show that the estimation error of WKL rapidly increases with the dimension.
The results are provided in~\cref{table:bad_WKL}.
In contrast, our method yields reasonable estimates in the same or similar cases presented in~\cref{figure:multidimensional_tests}.

\begin{table}[ht!]
    \caption{Evaluation of $ 5 $-NN weighted Kozachenko-Leonenko estimator on multidimensional uniformly distributed data.
    For each dimension $ d_X = d_Y $, $ 11 $ estimates of MI are acquired with the ground truth ranging from $ 0 $ to $ 10 $ with a fixed step.
    The RMSE is calculated for each set of estimates.}
    \label{table:bad_WKL}
    %\vskip 0.15in
    \begin{center}
    \begin{small}
    \begin{sc}
    \begin{tabular}{ccccccc}
        \toprule
        $ d_{X,Y} $ & $2$ & $4$ & $8$ & $16$ & $32$ & $64$  \\
        \midrule
        RMSE & $2.2$ & $1.0$ &$127.9$ & $227.5$ & $522.4$ & $336.2$  \\
        \bottomrule
    \end{tabular}
    \end{sc}
    \end{small}
    \end{center}
    %\vskip -0.1in
\end{table}

Overall, the proposed estimator performs well during all the experiments,
including the incompressible high-dimensional data, large MI values and non-Gaussian-based tests.

\section{Discussion}
\label{section:discussion}

Information-theoretic analysis of deep neural networks is a novel and developing approach.
As it relies on a well-established theory of information, it potentially can provide fundamental, robust and intuitive results~\cite{shwartz_ziv2017opening_black_box, he2023generalization_bounds}.
Currently, this analysis is complicated due to main information-theoretic qualities~--- \emph{differential entropy} and \emph{mutual information}~--- being hard to measure in the case of high-dimensional data.

We have shown that it is possible to modify the conventional normalizing flows setup to harness all the benefits of
%these generative models and
simple and robust closed-form expressions for mutual information.
Non-asymptotic error bounds for both variants of our method are derived, asymptotic variance and consistency analysis is carried out.
We provide useful theoretical and practical insights on using the proposed method effectively.
We have demonstrated the effectiveness of our estimator in various settings, including compressible and incompressible high-dimensional data, high values of mutual information and the data not acquired from the Gaussian distribution via invertible mappings.

%Although convergence of the proposed estimator follows from the convergence analysis of conventional normalizing flows if $ (X,Y) $ \emph{can} be jointly Gaussionized via separate mappings,
%convergence of the proposed method in the general case is still an open question.
%We consider this to be the major limitation of our estimator.
%However, conducted experiments show good performance of our method even in the case of data,
%which appear in the literature as non-Gaussian-based examples~\cite{czyz2023beyond_normal}.

Finally, it is worth noting that despite normalizing flows and Gaussian base distributions being used throughout our work,
the proposed method can be extended to any type of base distribution with closed-form expression for mutual information and to any injective generative model. % with tractable Jacobians.
For example, a subclass of diffusion models can be considered~\cite{zhang2021diffusion_normalizing_flows, kingma2021density_estimation_with_diffusion_models}.
Injective normalizing flows~\cite{teng2019invertible_AE, brehmer2020flows_manifold_learning} are also compatible with the proposed pipeline.
Gaussian mixtures can also be used as base distributions due to a relatively cheap MI calculation and the universality property~\cite{czyz2023pointwise_MI}.

\paragraph{Limitations} The main limitation of the general method is the ambiguity of $ \mathcal{Q} $ (the family of PDF estimators used to estimate MI in the latent space),
which can can be either rich (yielding a consistent, but possibly expensive estimator), or poor (leading to the inconsistency of the estimate).
However, in~\cite{czyz2023pointwise_MI} it is argued that mixture models can achieve rather good tradeoff between the quality and the cost of a PMI approximation.

The major limitation of $ \normal $-MIENF is that its consistency is proven only for a certain class of distributions: the probability distribution should be equivalent to a Gaussian via a Cartesian product of diffeomorphisms.
However, mathematical simplicity, rigorous bounds, low variance and relative practical success of the estimator suggest that the proposed method achieves a decent tradeoff.

\newpage
\bibliography{references}

\newpage
\appendix
%\onecolumn

\section{Complete proofs}
\label{appendix:proofs}

\newcommand*{\INAPPENDIX}{}

\printProofs

\section{Non-Gaussian-based tests}
\label{appendix:non_gaussian_based_tests}

As our estimator is based on Gaussianization,
it seems natural that we observe good performance in the experiments with synthetic data acquired from the correlated Gaussian vectors via invertible transformations.
Possible bias towards such data can not be discriminated via the \emph{independency} and \emph{self-consistency} tests, and hard to discriminate via the \emph{data-processing} test proposed in~\cite{czyz2023beyond_normal, franzese2024minde} for the following reasons:

\begin{enumerate}
    \item
        \emph{Independency} test requires $ \hat I(X;Y) \approx 0 $ for independent $ X $ and $ Y $.
        In such case, as $ \cov(f_X(X), f_Y(Y)) = 0 $ for any measurable $ f_X $ and $ f_Y $,
        $ \hat I_{\textnormal{MIENF}}(X;Y) \approx 0 $ in any meaningful scenario (no overfitting, no ill-posed transformations),
        regardless of the marginal distributions of $ X $ and $ Y $.
        %As $ X $ and $ Y $ can always be Gaussianized independently~\cite{chen2000gaussianization},
        %and will still be independent after the Gaussianization,
        %the latent (Gaussianized) vectors $ \xi = f_X(X) $ and $ \eta = f_Y(\eta) $ will be jointly Gaussian.
        %This, of course, means that our estimator will yield $ \hat I(X;Y) \approx 0 $, as $ \xi $ and $ \eta $
    \item
        \emph{Self-consistency} test requires $ \hat I(X;Y) \approx \hat I(g(X);Y) $ for $ X $, $ Y $ and $ g $ satisfying \cref{theorem:MI_under_nonsingular_mappings}.
        In our setup, this test only measures the ability of normalizing flows to invert $ g $,
        and provides no information about the quality of $ \hat I(X;Y) $ and $ \hat I(g(X);Y) $.
        
        Moreover, as we leverage~\cref{algorithm:dataset_generation_and_estimator_evaluation} with the Gaussian base distribution for the dataset generation, we somewhat test our estimator for the self-consistency.
    \item
        \emph{Data-processing} test leverages the \emph{data processing inequality}~\cite{cover2006information_theory} via requiring $ \hat I(X;Y) \geq \hat I(g(X);Y) $ for any $ X $, $ Y $ and measurable $ g $.
        Theoretically, this test may highlight the bias of our estimator towards binormalizable data.
        However, this requires constructing $ X $, $ Y $ and $ g $,
        so $ X $ and $ Y $ are not binormalizable, $ g(X) $ and $ Y $ are and $ \hat I(X;Y) < \hat I(g(X);Y) $,
        which seems challenging to achieve.
\end{enumerate}

That is why we use two additional, non-Gaussian-based families of distributions with known closed-form expressions for MI and easy sampling procedures: \emph{multivariate Student distribution}~\cite{arellano_valle2013MI_for_skew_distributions} and \emph{smoothed uniform distribution}~\cite{czyz2023beyond_normal}.

In the following subsections, we provide additional information about the distributions,
closed-form expressions for MI and sampling procedures.

\subsection{Multivariate Student distribution}
\label{subsection:multivariate_Student_distribution}

%The random vectors $ (X,Y) $ of dimensions $ n $ and $ m $ respectively with the desired mutual information $ I(X;Y) $ can be acquired as follows.

Consider $ (n+m) $-dimensional $ (\tilde X; \tilde Y) \sim \normal(0, \Sigma) $, where $ \Sigma $ is selected to achieve $ I(\tilde X; \tilde Y) = \varkappa > 0 $.
Firstly, a correction term is calculated in accordance to the following formula:
\[
    c(k, n, m) = f(k) + f(k + n + m) - f(k + n) - f(k + m),
    \quad
    f(x) = \log \Gamma \left( \frac{x}{2} \right) - \frac{x}{2} \psi \left( \frac{x}{2} \right),
\]
where $ k $ is the number of degrees of freedom, $ \psi $ is the digamma function.
Secondly, $ X = \tilde X / \sqrt{k / U} $ and $ Y = \tilde Y / \sqrt{k / U} $ are defined, where $ U \sim \chi_k^2 $.
The resulting vectors are distributed according to the multivariate Student distribution with $ k $ degrees of freedom.
According to~\cite{arellano_valle2013MI_for_skew_distributions}, $ I(X;Y) = \varkappa + c(k,n,m) $.
During the generation, $ \varkappa $ is set to $ I(X;Y) - c(k,n,m) $ to achieve the desired value of $ I(X;Y) $.

Note that $ I(X;Y) \neq 0 $ even in the case of independent $ \tilde X $ and $ \tilde Y $,
as some information between $ X $ and $ Y $ is shared via the magnitude.%~\cite{czyz2023beyond_normal}.

\subsection{Smoothed uniform distribution}
\label{subsection:smoothed_uniform_distribution}

\begin{lemma}
    \label{lemma:closed_form_MI_expression_smoothed_uniform}
    Consider independent $ X \sim \uniform[0;1] $, $ Z \sim \uniform[-\varepsilon; \varepsilon] $ and $ Y = X + Z $.
    Then
    \begin{equation}
        \label{eq:smoothed_uniform_mutual_information}
        I(X; Y) =
        \begin{dcases}
            \varepsilon - \log (2 \varepsilon), &\;\; \varepsilon < 1/2 \\
            (4 \varepsilon)^{-1} &\;\; \varepsilon \geq 1/2
        \end{dcases}
    \end{equation}
\end{lemma}

\begin{proof}
    Probability density function of $ Y $ (two cases):
    %\[
    %    \PDF_Y(y) = (\PDF_X * \PDF_Z)(y) =
    %    \begin{dcases}
    %        \begin{cases}
    %            0, &\; y < -\varepsilon \vee y > 1 + \varepsilon \\
    %            \frac{y + \varepsilon}{2\varepsilon}, & \; -\varepsilon \leq y < \varepsilon \\
    %            1, & \; 0 \leq y < 1 \\
    %            \frac{1 + \varepsilon - y}{2\varepsilon}, & \; 1 - \varepsilon \leq y < 1 + \varepsilon \\
    %        \end{cases} &\quad (\varepsilon < 1/2) \\
    %        \begin{cases}
    %            0, &\; y < -\varepsilon \vee y > 1 + \varepsilon \\
    %            \frac{y + \varepsilon}{2\varepsilon}, & \; -\varepsilon \leq y < 1 - \varepsilon \\
    %            \frac{1}{2\varepsilon}, & \; 1 - \varepsilon \leq y < \varepsilon \\
    %            \frac{1 + \varepsilon - y}{2\varepsilon}, & \; \varepsilon \leq y < 1 + \varepsilon \\
    %        \end{cases} &\quad (\varepsilon \geq 1/2)
    %    \end{dcases}
    %\]
    \[
        (\varepsilon < 1/2): \qquad
        \PDF_Y(y) = (\PDF_X * \PDF_Z)(y) =
        \begin{cases}
            0, &\; y < -\varepsilon \vee y \geq 1 + \varepsilon \\
            \frac{y + \varepsilon}{2\varepsilon}, & \; -\varepsilon \leq y < \varepsilon \\
            1, & \; \varepsilon \leq y < 1 - \varepsilon \\
            \frac{1 + \varepsilon - y}{2\varepsilon}, & \; 1 - \varepsilon \leq y < 1 + \varepsilon \\
        \end{cases}
    \]
    \[
        (\varepsilon \geq 1/2): \qquad
        \PDF_Y(y) = (\PDF_X * \PDF_Z)(y) =
        \begin{cases}
            0, &\; y < -\varepsilon \vee y \geq 1 + \varepsilon \\
            \frac{y + \varepsilon}{2\varepsilon}, & \; -\varepsilon \leq y < 1 - \varepsilon \\
            \frac{1}{2\varepsilon}, & \; 1 - \varepsilon \leq y < \varepsilon \\
            \frac{1 + \varepsilon - y}{2\varepsilon}, & \; \varepsilon \leq y < 1 + \varepsilon \\
        \end{cases}
    \]
    Differential entropy of a uniformly distributed random variable:
    \[
        h(U[a;b]) = \log(b - a)
    \]
    Conditional differential entropy of $ Y $ with respect to $ X $:
    \[
        h(Y \mid X) = \expect_{x \sim X} h(Y \mid X = x) = \expect_{x \sim X} h(Z + x \mid X = x)
    \]
    As $ X $ and $ Z $ are independent,
    \begin{equation}
        \label{eq:smoothed_uniform_Y_conditional_entropy}
        \expect_{x \sim X} h(Z + x \mid X = x) = \expect_{x \sim X} h(Z + x) = \int\limits_0^1 \log (2 \varepsilon) \, dx = \log (2 \varepsilon)
    \end{equation}
    Differential entropy of $ Y $:
    \begin{equation}
        \label{eq:smoothed_uniform_Y_entropy}
        h(Y) = - \int\limits_{-\infty}^\infty \PDF_Y(y) \, dy =
        \begin{cases}
            \varepsilon, &\;\; \varepsilon < 1/2 \\
            (4\varepsilon)^{-1} + \log(2 \varepsilon), &\;\; \varepsilon \geq 1/2
        \end{cases}
    \end{equation}
    The final result is acquired via substituting~\eqref{eq:smoothed_uniform_Y_conditional_entropy} and~\eqref{eq:smoothed_uniform_Y_entropy} into~\eqref{eq:mutual_information_from_conditional_entropy}.
\end{proof}

Equation~\eqref{eq:smoothed_uniform_mutual_information} can be inverted:
\begin{equation}
    \label{eq:smoothed_uniform_mutual_information_inverted}
    \varepsilon =
    \begin{cases}
        -W\left[ -\frac{1}{2}\exp(-I(X;Y))\right], &\;\; I(X;Y) < 1/2 \\
        (4 \cdot I(X;Y))^{-1}, &\;\; I(X;Y) \geq 1/2
    \end{cases},
\end{equation}
where $ W $ is the product logarithm function.

\section{Technical details}
\label{appendix:technical_details}

In this section, we describe the technical details of our experimental setup:
architecture of the neural networks, hyperparameters, etc.

For the tests described in Section~\ref{section:experiments}, we use architectures listed in Table~\ref{table:synthetic_architecture}.
For the flow models, we use the \texttt{normflows} package~\cite{stimper2023normflows}.
The autoencoders are trained via Adam~\cite{kingma2017adam} optimizer
on $ 5 \cdot 10^3 $ images with a batch size $ 5 \cdot 10^3 $,
a learning rate $ 10^{-3} $ and MAE loss for $ 2 \cdot 10^3 $ epochs.
The MINE/NWJ/Nishiyama critic network is trained via the Adam optimizer
on $ 5 \cdot 10^3 $ pairs of images with a batch size $ 512 $,
a learning rate $ 10^{-3} $ for $ 5 \cdot 10^3 $ epochs.
The GLOW normalizing flow is trained via the Adam optimizer
on $ 10 \cdot 10^3 $ images with a batch size $ 1024 $,
a learning rate decaying from $ 5 \cdot 10^{-4} $ to $ 1 \cdot 10^{-5} $ for $ 2 \cdot 10^3 $ epochs.
Nvidia Titan RTX is used to train the models.
In any setup, each experiment took no longer than one hour to be completed.

\begin{table}[ht!]
    \center
    \caption{The NN architectures used to conduct the tests in Section~\ref{section:experiments}.}
    \label{table:synthetic_architecture}
    \begin{tabular}{cc}
    \toprule
    NN & Architecture \\
    \midrule
    \makecell{AEs,\\ $ 16 \times 16 $ ($ 32 \times 32 $) \\ images} &
        \small
        \begin{tabular}{rl}
            $ \times 1 $: & Conv2d(1, 4, ks=3), BatchNorm2d, LeakyReLU(0.2), MaxPool2d(2) \\
            $ \times 1 $: & Conv2d(4, 8, ks=3), BatchNorm2d, LeakyReLU(0.2), MaxPool2d(2) \\
            $ \times 2(3) $: & Conv2d(8, 8, ks=3), BatchNorm2d, LeakyReLU(0.2), MaxPool2d(2) \\
            $ \times 1 $: & Dense(8, dim), Tanh, Dense(dim, 8), LeakyReLU(0.2) \\
            $ \times 2(3) $: & Upsample(2), Conv2d(8, 8, ks=3), BatchNorm2d, LeakyReLU(0.2) \\
            $ \times 1 $: & Upsample(2), Conv2d(8, 4, ks=3), BatchNorm2d, LeakyReLU(0.2) \\
            $ \times 1 $: & Conv2d(4, 1, ks=3), BatchNorm2d, LeakyReLU(0.2) \\
        \end{tabular} \\
        \\
    \makecell{MINE, critic NN, \\ $ 16 \times 16 $ ($ 32 \times 32 $) \\ images} &
        \small
        \begin{tabular}{rl}
            $ \times 1 $: & [Conv2d(1, 16, ks=3), MaxPool2d(2), LeakyReLU(0.01)]$^{ \times 2 \; \text{in parallel}} $ \\
            $ \times 1(2) $: & [Conv2d(16, 16, ks=3), MaxPool2d(2), LeakyReLU(0.01)]$^{ \times 2 \; \text{in parallel}} $ \\
            $ \times 1 $: & Dense(256, 128), LeakyReLU(0.01) \\
            $ \times 1 $: & Dense(128, 128), LeakyReLU(0.01) \\
            $ \times 1 $: & Dense(128, 1) \\
        \end{tabular} \\ \\
    \makecell{GLOW, \\ $ 16 \times 16 $ ($ 32 \times 32 $) \\ images} &
    \small
        \begin{tabular}{rl}
            $ \times 1 $: & \makecell[l]{$ 4 $ ($ 5 $) splits, $ 2 $ GLOW blocks between splits, \\ $ 16 $ hidden channels in each block, leaky constant $ = 0.01 $} \\
            $ \times 1 $: & Orthogonal linear layer \\
            $ \times 4 $: & RealNVP(AffineCouplingBlock(MLP($d / 2 $, $ 32 $, $ d $)), Permute-swap) \\
        \end{tabular} \\ \\
    \makecell{RealNVP, \\ $ d $-dimensional data} &
    \small
        \begin{tabular}{rl}
            $ \times 6 $: & RealNVP(AffineCouplingBlock(MLP($d / 2 $, $ 64 $, $ d $)), Permute-swap) \\
        \end{tabular} \\
    \bottomrule
    \end{tabular}
\end{table}

Here we do not explicitly define $ g_\xi $ and $ g_\eta $
used in the tests with synthetic data,
as these functions smoothly map low-dimensional vectors to high-dimensional images and, thus, are very complex.
A Python implementation of the functions in question is available in the supplementary material,
see the directory \texttt{mutinfo/source/python/mutinfo/distributions/images}.

\end{document}